\documentclass[journal]{IEEEtran}
\usepackage{amsmath}
\usepackage{xcolor}
\usepackage{amssymb}
\usepackage{amsthm}
\usepackage{times}
\usepackage{graphicx}
\usepackage{color}
\usepackage{multirow}
\usepackage{subfigure}
\usepackage{cite}
\usepackage{rotating}
\usepackage{bbm}
\usepackage{latexsym}
\usepackage{algorithm}
\usepackage{algorithmic}
\usepackage{threeparttable}
\ifCLASSINFOpdf
\else
\fi
\begin{document}
%
\title{Smoothing Graphons for Modelling Exchangeable Relational Data}
%
%
%
\author{Yaqiong~Li$^*$, 
        Xuhui~Fan$^*$, 
        Ling~Chen,
        Bin~Li, 
        and~Scott~A.~Sisson~
\thanks{Yaqiong Li and Xuhui Fan are co-first authors.}
\thanks{Yaqiong Li and Ling Chen are with the Centre for Artificial Intelligence, Faculty of Engineering and Information Technology, University of Technology Sydney, Ultimo, NSW 2007, Australia.}
\thanks{Xuhui Fan and Scott A.~Sisson are based at the School of Mathematics and Statistics, University of New South Wales, Sydney, NSW 2052, Australia.}
\thanks{Bin Li is with the School of Computer Science, Fudan University, Shanghai, China.}
}

\markboth{IEEE transactions on neural networks and learning systems,~Vol.~XX, No.~X, XXX~XXXX}%
{Shell \MakeLowercase{\textit{et al.}}: Bare Demo of IEEEtran.cls for Journals}
%



\maketitle

\renewcommand{\algorithmicrequire}{\textbf{Input:}}
\renewcommand{\algorithmicensure}{\textbf{Output:}}
\newtheorem{property}{Property}
\newtheorem{lemma}{Lemma}

\newtheorem{theorem}{Theorem}
\newtheorem{proposition}{Proposition}
\begin{abstract}
Modelling exchangeable relational data can be 
described by \textit{graphon theory}. Most  Bayesian methods for modelling exchangeable relational data can be attributed to this framework by exploiting different forms of graphons. However, the graphons adopted by existing Bayesian methods are either piecewise-constant functions, 
which are insufficiently flexible for accurate modelling of the relational data, 
or are complicated continuous functions, which incur heavy computational costs for inference. In this work, we introduce a smoothing procedure to  piecewise-constant graphons to form {\em smoothing graphons}, which 
permit continuous intensity values for describing relations, but without impractically increasing computational costs.
In particular, we focus on the Bayesian Stochastic Block Model (SBM) 
and demonstrate how to adapt the piecewise-constant SBM graphon to the smoothed version.
We initially propose the Integrated Smoothing Graphon (ISG) which introduces one smoothing parameter to the SBM graphon to generate continuous relational intensity values.
We then develop the Latent Feature Smoothing Graphon (LFSG), which improves on the ISG by introducing auxiliary hidden labels to decompose the calculation of the ISG intensity and enable efficient inference.  Experimental results on real-world data sets validate the advantages of applying smoothing strategies to the Stochastic Block Model,
demonstrating that smoothing graphons can greatly improve AUC and precision for link prediction
without increasing computational complexity.

\end{abstract}


%
\IEEEpeerreviewmaketitle

\section{Introduction}
Exchangeable relational data~\cite{nowicki2001estimation,ishiguro2010dynamic,nonpa2013schmidt}, such as tensor data~\cite{tnnls_8497036,pensky2019dynamic} and collaborative filtering data~\cite{tnnls_7112169,tnnls_8525418,Li_transfer_2009}, are commonly observed in many real-world applications. In general, exchangeable relational data describe the relationship between two or more nodes~(e.g.~friendship linkages in social networks; user-item rating matrices in recommendation systems; 
and protein-to-protein interactions in computational biology), where exchangeability refers to the phenomenon that the joint distribution over all observed relations remains invariant under node permutations. 
Techniques for modelling exchangeable relational data 
include node partitioning to form ``homogeneous blocks"~\cite{nowicki2001estimation,kemp2006learning,roy2009mondrian}, graph embedding methods to generate low-dimensional representations~\cite{tnnls_6508899,tnnls_6842607,tnnls_8587135}, and optimization strategies to minimize prediction errors~\cite{tnnls_6208890,tnnls_8440680}. 

\textit{Graphon theory}~\cite{orbanz2009construction, orbanz2014bayesian,RandomFunPriorsExchArrays} has recently recently been proposed as a unified theoretical framework for modelling exchangeable relational data. In graphon theory, each relation from a node $i$ to another node $j$ is represented by an \textit{intensity} value generated by a \textit{graphon function}, which maps from the corresponding coordinates of the node pair in a unit square, $(u_{i}^{(1)}, u_{j}^{(2)})$, to an intensity value in a unit interval. Many existing Bayesian methods for modelling exchangeable relational data can be described using graphon theory with various graphon functions~(see Fig.~\ref{fig:graphon_comparison}). Representative models include the Stochastic Block Model (SBM)~\cite{nowicki2001estimation,kemp2006learning}, the Mondrian Process Relational Model (MP-RM)~\cite{roy2009mondrian},  the Rectangular Tiling Process Relational Model (RTP-RM)~\cite{nakano2014rectangular}, 
the Rectangular Bounding Process Relational Model~(RBP-RM)~\cite{NIPS2018_RBP} and the Gaussian Process Prior Relational Model~(GP-RM)~\cite{orbanz2009construction}. 

These existing models can be broadly classified into two categories. The first category, which includes the SBM, MP-RM, RTP-RM and RBP-RM models, uses node-partitioning strategies to construct the relational model. 
By partitioning the set of nodes into groups along node co-ordinate margins, blocks can be constructed from these marginal groups that partition the full-dimensional co-ordinate space according to a given construction method (Fig.~\ref{fig:graphon_comparison}).
These models then assume that the relation intensity for node pairs is constant within each block. That is, the graphon function that generates intensity values over node co-ordinate space is constructed in a piecewise-constant manner. However, such piecewise-constant graphons can only provide limited modelling flexibility with a fixed and constant number of intensity values~(e.g.~equivalent to the number of blocks). 
As a result, they are restricted in their ability to model the ground-truth well.
The second category of relational models, which includes the GP-RM, 
aims to address this limitation as the graphon function can provide continuous intensity values. 
However, the computational complexity for estimating this graphon function is proportional to the cube of the number of nodes, which makes it practically non-viable for medium or large sized datasets. 


In this paper, we propose to apply a smoothing procedure to  piecewise-constant graphons to form \textit{smoothing graphons}, which will naturally permit  continuous intensity values for describing relations without impractically increasing  computational costs. As the Stochastic Block Model is one of the most popular Bayesian methods for modelling exchangeable relational data, we focus on developing  smoothing strategies within the piecewise-constant SBM graphon framework. In particular, we develop two variant smoothing strategies for the SBM: the Integrated Smoothing Graphon~(ISG) and the Latent Feature Smoothing Graphon~(LFSG). 

\begin{figure*}[ht]
\centering
\includegraphics[width = 0.9\textwidth]{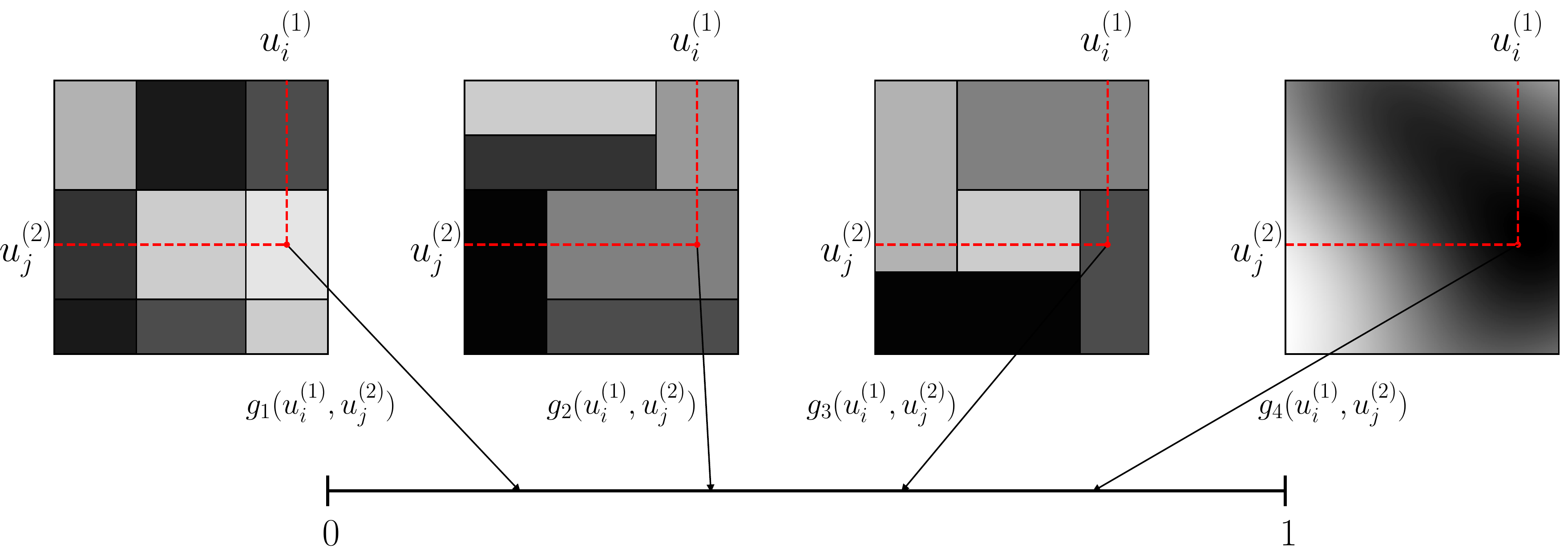}
\caption{Visualisation of Bayesian graphon-construction methods for modeling exchangeable relational data. 
From left to right: the Stochastic Block Model~(SBM); the Mondrian Process Relational Model~(MP-RM); the Rectangular Tiling Process Relational Model~(RTP-RM) and the Gaussian Process Prior Relational Model~(GP-RM). For any pair of node coordinates $(u_{i}^{(1)}, u_{j}^{(2)})$  the relation intensity is mapped from the unit square to a unit interval using a graphon 
function~(denoted $g_1,\ldots,g_4$), 
where a darker colour observed in the unit square represents a higher mapped intensity in the unit interval.}
\label{fig:graphon_comparison}
\end{figure*}

\begin{itemize} 
    \item ISG: In contrast to existing piecewise-constant graphons, which determine the intensity value based only on the  block within which a node pair resides, 
    the ISG alternatively calculates a mixture intensity for each pair of nodes by taking into account the intensities of all other blocks. The resulting mixture graphon function is constructed so that its output values are continuous. 
    \item LFSG: This strategy introduces auxiliary pairwise hidden labels to decompose the calculation of the mixture intensity used in the ISG, in order to enables efficient inference. 
    In addition, the introduction of these labels allows each node to belong to multiple groups in each dimension (e.g.~a user may interact with different people by playing different roles in a social network), which provides more modelling flexibility compared with the ISG (and existing piecewise-graphons) where each node is assigned to one group only. 
\end{itemize}


Note that while we develop the ISG and LFSG for SBM-based graphons, our smoothing approach can easily be applied to other piecewise-constant graphons. The main contributions of our work are summarised as follows:
\begin{itemize}
    \item We identify the key limitation of existing piecewise-constant graphons and develop a smoothing strategy to flexibly generate continuous graphon intensity values, which may better reflect the reality of a process.
    \item We develop the ISG smoothing strategy for the SBM to demonstrate how piecewise-constant graphons can be converted into smoothing graphons. 
    \item We improve on the ISG by devising the LFSG, which achieves the same objective of generating continuous intensity values but without sacrificing computation efficiency. Compared with the ISG where each node belongs to only one group, the LFSG allows each node to belong to multiple groups (e.g.~so that the node plays different roles in different relations), and thereby also providing a probabilistic interpretation of node groups.
    \item We evaluate the performance of our methods on the task of link prediction by comparing with the SBM and other benchmark methods. The experimental results clearly show that the smoothing graphons can achieve significant performance improvement over piecewise-constant graphons. 
\end{itemize}

\section{Preliminaries}
\subsection{Graphon theory}
The Aldous--Hoover theorem~\cite{hoover1979relations,aldous1981representations} provides the theoretical foundation for modelling exchangeable multi-dimensional arrays~(i.e.~exchangeable relational data) conditioned on a stochastic partition model. A random $2$-dimensional array is called {\em separately exchangeable} if its distribution is invariant under separate permutations of rows and columns.
\begin{theorem}~\cite{orbanz2014bayesian, RandomFunPriorsExchArrays}: \label{theoremgraphon}
  A random array $(R_{ij})$ is separately exchangeable if and only if it can be represented as follows: there exists a random measurable function $F:[0,1]^3 \mapsto \mathcal{X}$ such that $(R_{ij}) \overset{d}{=} \left(F(u_i^{(1)}, u_j^{(2)}, \nu_{ij})\right)$, where $\{u_i^{(1)}\}_i, \{u_j^{(2)}\}_j$ and $\{\nu_{ij}\}_{i,j}$ are two sequences and an array of i.i.d.~uniform random variables in $[0,1]$, respectively.
\end{theorem}

Many existing Bayesian methods for modelling exchangeable relational data can be represented as in Theorem \ref{theoremgraphon}, using specific forms of the mapping function $F$. For instance, as illustrated in Fig.~\ref{fig:graphon_comparison}, given the uniformly distributed node coordinates $(u_i^{(1)}, u_j^{(2)})$, the SBM corresponds to a regular-grid constant graphon; the MP-RM uses a $k$-d tree structured constant graphon; the RTP-RM adopts an arbitrary rectangle constant graphon; and the GP-RM induces a continuous $2$-dimensional function. While taking different forms, these graphon functions commonly map from pairs of node coordinates in a unit square to intensity values in a unit interval. As shown in Fig.~\ref{fig:graphon_comparison}, the darker colour at the pair of node coordinates the higher intensity in the interval, which corresponds to a larger probability of observing or generating the relation between the pair of nodes. 

\subsection{Piecewise-constant graphons and their limitations}
Many alternative piecewise-constant graphons can be implemented to model exchangeable relational data $R$, where $R$ is a binary adjacency matrix which can be either directed~(asymmetric) or undirected~(symmetric). Here we consider the more complicated situation where $R$ is a $n\times n$ asymmetric matrix with $R_{ji}\neq R_{ij}$ (the extension of our method to the symmetric case is straightforward). For any two nodes in $R$, if node $i$ is related to node $j$ then $R_{ij}=1$, otherwise $R_{ij}=0$.

We take the SBM as an illustrative example.
In a two-dimensional SBM, there are two  distributions generating the groups, $\pmb{\theta}^{(1)},\pmb{\theta}^{(2)}\sim\text{Dirichlet}(\pmb{\alpha}_{1\times K})$, where $K$ is the number of groups and $\pmb{\alpha}_{1\times K}$ is the $K$-vector concentration parameter.
Each node  $i\in \{1, \ldots, n\}$ is associated with two hidden labels $z_i^{(1)}, z_i^{(2)}\in\{1, \ldots, K\}$, and $\{z_i^{(1)}\}_{i}\sim\text{Categorical}(\pmb{\theta}^{(1)}), \{z_i^{(2)}\}_{i}\sim\text{Categorical}(\pmb{\theta}^{(2)})$ for $i=1,\ldots,n$. Hence, $z_i^{(1)}$ and $z_i^{(2)}$ denote the particular groups that node $i$ belongs to in  two dimensions, respectively (that is, $z_i^{(1)}$ is the group of node $i$ when $i$ links to other nodes, and  $z_i^{(2)}$ is the group of node $i$ when other nodes link to it). The relation $R_{ij}$ from node $i$ to node $j$ is then generated based on the interaction between their respective groups $z_i^{(1)}$ and $z_j^{(2)}$.  

Let  $\pmb B$ be a $K\times K$ matrix, where each entry $B_{k_1, k_2}\in[0, 1]$ denotes the probability of generating a relation from group $k_1$ in the first dimension to group $k_2$ in the second dimension. For $k_{1},k_{2}=1, \ldots, K, B_{k_{1},k_{2}}\sim\text{Beta}(\alpha_0, \beta_0)$, where $\alpha_0, \beta_0$ are hyper-parameters for $\{B_{k_1,k_2}\}_{k_1, k_2}$. That is, we have  $P(R_{ij}=1|z_i^{(1)}, z_j^{(2)}, \pmb B)=B_{z_i^{(1)}, z_j^{(2)}}$. 

Now consider the SBM from the graphon perspective (Fig.~\ref{fig:graphon_comparison}; left). Let $\pmb{\theta}^{(1)},\pmb{\theta}^{(2)}$ be group (or segment) distributions of the two dimensions in a unit square respectively. The generation of hidden labels $z_i^{(1)}$ for node $i$ and $z_j^{(2)}$ for node $j$ proceeds as follows:  Uniform random variables $u_i^{(1)}$ and $u_j^{(2)}$ are respectively generated in the first and second dimensions. Then, $z_i^{(1)}$ and $z_j^{(2)}$ can be determined by checking in which particular segments of $\pmb{\theta}^{(1)}$ and $\pmb{\theta}^{(2)}$, $u_i^{(1)}$ and $u_j^{(2)}$  are respectively located. Formally, we have:
\begin{align} \label{eq:SBM_partition_geneartion}
& \pmb{\theta}^{(1)}, \pmb{\theta}^{(2)}\sim\text{Dirichlet}(\pmb{\alpha}_{1\times K}),\:\: \:\:u_i^{(1)}, u_j^{(2)}\sim\text{Unif}[0, 1] \nonumber \\
&z_i^{(1)}=(\pmb{\theta}^{(1)})^{-1}(u_i^{(1)}),\:\: \:\: z_j^{(2)}=(\pmb{\theta}^{(2)})^{-1}(u_j^{(2)}),
\end{align}
where $(\pmb{\theta}^{(1)})^{-1}(u_i^{(1)})$ and $(\pmb{\theta}^{(2)})^{-1}(u_j^{(2)})$ respectively map $u_i^{(1)}$ and $u_j^{(2)}$ to particular segments of $\pmb{\theta}^{(1)}$ and $\pmb{\theta}^{(2)}$.

A regular-grid partition~($\boxplus$) can be formed in the unit square by combining the segment distributions $\pmb{\theta}^{(1)}, \pmb{\theta}^{(2)}$ in two dimensions. Each block in this regular-grid partition is presented in a rectangular shape. Let $L_{k}^{(1)}=\sum_{k'=1}^k{{\theta}}_{k'}^{(1)}$ {and $L_{k}^{(2)}=\sum_{k'=1}^k{{\theta}}_{k'}^{(2)}$ be the accumulated sum of the first $k$ elements of $\pmb{\theta}^{(1)}$ and $\pmb{\theta}^{(2)}$ respectively~(w.l.o.g.~$L_{0}^{(1)}=L_{0}^{(2)}=0, L_{K}^{(1)}=L_{K}^{(2)}=1$).}  Use $\Box_{k_1,k_2}=[L^{(1)}_{k_1-1}, L^{(1)}_{k_1}]\times [L^{(2)}_{k_2-1}, L^{(2)}_{k_2}]$ to represent the $(k_1, k_2)$-th block in the unit square of $[0,1]^2$, such that 
$\bigcup_{k_1, k_2}\Box_{k_1,k_2}=[0,1]^2$. Then, an intensity function defined on the pair $(u_i, u_j)$ can be obtained by the piecewise-constant graphon function 
\begin{align} \label{SBM_graphon_intensity_function}
g\left(u_i^{(1)}, u_j^{(2)}\right) = \sum_{k_1, k_2} \pmb{1}{((u_i^{(1)},u_j^{(2)})\in \Box_{k_1,k_2})} \cdot B_{k_1,k_2}
\end{align}
where $\pmb{1}(A)=1$ if $A$ is true and $0$ otherwise, and
where $B_{k_1, k_2}\in[0, 1]$ is the intensity of the $(k_1, k_2)$-th block. We term \eqref{SBM_graphon_intensity_function} the
SBM-graphon. Thus, the generative process of the SBM-graphon can be described as:
\begin{enumerate}
 \item For $k_{1},k_{2}=1, \ldots, K$, generate $B_{k_{1},k_{2}}\sim\text{Beta}(\alpha_0, \beta_0)$, where $\alpha_0, \beta_0$ are hyper-parameters;
\item Generate the segment distributions $\pmb{\theta}^{(1)}, \pmb{\theta}^{(2)}$ via Eq.~(\ref{eq:SBM_partition_geneartion}) and form the partition~($\boxplus$) according to combinations of $\pmb{\theta}^{(1)}, \pmb{\theta}^{(2)}$ in the unit square;
 \item Uniformly generate the $1^{st}$ dimension coordinates $\{u_i^{(1)}\}_{i=1}^n$ and the $2^{nd}$ dimension coordinates $\{u_i^{(2)}\}_{i=1}^n$ for all nodes;
 \item For $i,j=1, \ldots, n$
    \begin{enumerate}
        \item Calculate the intensity $g(u_i^{(1)}, u_j^{(2)})$ according to Eq.~(\ref{SBM_graphon_intensity_function}) based on the node coordinates $(u_i^{(1)}, u_j^{(2)})$;
        \item Generate $R_{ij}\sim\mbox{Bernoulli}(g(u_i^{(1)}, u_j^{(2)}))$.
    \end{enumerate}
\end{enumerate}

Alternatively, if considering the latent labels~($z_i^{(1)}, z_j^{(2)}$) for nodes $i$ and $j$, then from Eq.~(\ref{eq:SBM_partition_geneartion}), step 4) can also be written as
\begin{enumerate}
  \setcounter{enumi}{3}
  \item For $i, j=1, \cdots, n$, 
    \begin{enumerate}
        \item  Generate the latent labels~($z_i^{(1)}, z_j^{(2)}$) via Eq.~(\ref{eq:SBM_partition_geneartion});
        \item Generate $R_{ij}\sim\text{Bernoulli}\left(B_{z_i^{(1)}, z_j^{(2)}}\right)$.
    \end{enumerate}
\end{enumerate}

The SBM-graphon has several limitations.
Firstly, the SBM-graphon function 
(Eq.~(\ref{SBM_graphon_intensity_function})) 
is piecewise-constant. That is, the generated intensities for node pairs are discrete and the number of different intensity values is limited to the number of blocks in the partition ($\boxplus$). Consequently, this leads to an over-simplified description when modelling real
relational data, which can result in at least two issues. On the one hand, as long as two nodes belong to the same segment in one dimension, their probabilities of generating relations with another node are the same even if the distance between the two nodes in that dimension is quite far. Conversely, given two nodes that are  close in one dimension but belong to two adjacent segments, their probabilities of generating relations with another node could be dramatically different, depending on the respective block intensities (e.g., $B_{k_1,k_2}$). 

The second limitation of the SBM-graphon is that it determines the intensity value for a pair of nodes by considering only the block ($\Box_{k_1,k_2}$) in which $(u_i, u_j)$ resides. However, the nodes relations with other nodes, especially neighbouring nodes in adjacent blocks, may also be expected to have a certain influence on the generation of the target relation, if one considers the relational data collectively. As a result, it may perhaps be beneficial to consider the interactions that naturally exist among all blocks when generating the relation $R_{ij}$. 

The third limitation of the SBM-graphon is that it provides latent information of node clustering as a side-product through the hidden labels $\{z_i^{(1)}, z_i^{(2)}\}_{i=1}^n$. However, the clustering information may not be ideal because each node is assigned to only one cluster in each dimension. That is, when considering the outgoing relations from node $i$, it is assumed that node $i$ consistently plays one single role in any relation with other nodes. In fact, in practice a node may play different roles by participating in different relations with different nodes. As a result, it would be more useful and flexible to allow a node to belong to multiple clusters in each dimension.

To address the limitations of piecewise-constant graphons (and in particular, the SBM-graphon), we propose a smoothing strategy to enable piecewise-constant graphons to produce continuous intensity values. The proposed smoothing graphons naturally consider interactions between the partitions and allow each node to play multiple roles in different relations.

\section{Main Models}
\subsection{The Integrated Smoothing Graphon~(ISG)}
In order to improve on the limitations of the piecewise-constant graphon
we first  develop the Integrated Smoothing Graphon~(ISG), based on the SBM-graphon construction. 
The piecewise-constant nature of the SBM-graphon is created through the use of an indicator function in \eqref{SBM_graphon_intensity_function} that selects only the particular block accommodating the target node pair. 
Accordingly, we replace the indicator function with an alternative that can produce continuous intensity values. Moreover, to capture the interaction between all blocks, we construct the smoothing graphon function to generate the intensity value as a summation over all block intensities, weighted by the importance of each block. Let $\bar{F}_{\Box_{k_1,k_2}}(u_i^{(1)}, u_j^{(2)})$ be the weight of the block $\Box_{k_1,k_2}$ with respect to $(u_i^{(1)}, u_j^{(2)})\in[0, 1]^2$. The mixture intensity $g\left(u_i^{(1)}, u_j^{(2)}\right)$, used to determine the $R_{ij}$, can then be represented as
\begin{align} \label{smooth_graphon_intensity_function}
g\left(u_i^{(1)}, u_j^{(2)}\right) = \sum_{k_1, k_2}\bar{F}_{\Box_{k_1,k_2}}(u^{(1)}_i, u^{(2)}_j)\cdot B_{k_1, k_2},
\end{align}
where { $\sum_{k_1,k_2}\bar{F}_{\Box_{k_1,k_2}}(u^{(1)}_i, u^{(2)}_j)=1$.} 

The ISG generative process can be summarised as:

1)$\sim$3) The block intensities~($\pmb B$), graphon partition~($\boxplus$) and $2$-dimensional coordinates ($\{u_{i}^{(1)}, u_{i}^{(2)}\}_{i=1}^n$) are generated as for the SBM-graphon;
\begin{enumerate}
  \setcounter{enumi}{3}
  \item For $i, j=1, \cdots, n$, 
    \begin{enumerate}
        \item Calculate the mixture intensity $g\left(u_i^{(1)}, u_j^{(2)}\right)$ according to \eqref{smooth_graphon_intensity_function} for the node coordinates $(u_i^{(1)}, u_j^{(2)})$;
        \item Generate $R_{ij}\sim\mbox{Bernoulli}\left(g\left(u_i^{(1)}, u_j^{(2)}\right)\right)$.
    \end{enumerate}
\end{enumerate}

As a consequence, while the SBM-graphon determines the relation intensity based only on the single block in which $(u_i^{(1)}, u_j^{(2)})$ resides, the ISG computes a mixture intensity as a weighted (and normalised) sum of all block intensities. That is, instead of assigning a weight of $1$ for one particular block and weights of $0$ for all other blocks, the ISG weights the importance of each block with respect to the pair of node coordinates $(u_i^{(1)}, u_j^{(2)})$. As long as the weighting function $\bar{F}_{\Box_{k_1,k_2}}(u_i^{(1)}, u_j^{(2)})$ is continuous, then the mixture intensity \eqref{smooth_graphon_intensity_function} is also continuous. The intensity function \eqref{smooth_graphon_intensity_function} then becomes a smoothing graphon function.

The ISG allows the mixture intensity to take any value between the minimum and maximum of all block intensities. As a result, the ISG provides more modelling flexibility compared to the SBM-graphon, where only limited discrete intensity values~(equivalent to the number of blocks) are available to describe relations.

\subsection{Construction of the mixture intensity}
To ensure that the graphon function \eqref{smooth_graphon_intensity_function} is continuous, we consider an integral-based weighting function of the form
{\small\begin{align}
{F}_{\Box_{k_1,k_2}}(u_i^{(1)}, u_j^{(2)}) \propto
\int_{L^{(1)}_{k_1-1}}^{L^{(1)}_{k_1}}f(x-u_i^{(1)})dx 
\cdot {\int_{L^{(2)}_{k_2-1}}^{L^{(2)}_{k_2}}f(x-u_j^{(2)})dx}
\end{align}}
where $f(x-u)$ is a 
univariate derivative function. Beyond the continuity requirement,  $f(x-u)$ and ${F}_{\Box_{k_1,k_2}}(u_i^{(1)}, u_j^{(2)})$ should satisfy the following three conditions: 
\begin{enumerate}
    \item $f(x-u)$ is non-negative;
    \item $f(x-u)$ increases with decreasing distance between $x$ and the corresponding coordinate $u$ (i.e.~$|x-u_i^{(1)}|$ or $|x-u_j^{(2)}|$). 
    This condition means that the closer the block $\Box_{k_1,k_2}$ is to the pair of node coordinates $(u_i^{(1)}, u_j^{(2)})$, the larger weight the block will be assigned. The maximum weight value is achieved when $|x-u_i^{(1)}|=0$ and  $|x-u_j^{(2)}|=0$; 
    \item 
    The total weight of all blocks remains invariant regardless of different partitioning of the unit space. That is,  ${F}_{\Box_{k_1,k_2}}(u_i^{(1)}, u_j^{(2)})={F}_{\Box'_{k_1,k_2}}(u_i^{(1)}, u_j^{(2)})+{F}_{\Box''_{k_1,k_2}}(u_i^{(1)}, u_j^{(2)})$, where $\Box'_{k_1,k_2}, \Box''_{k_1,k_2}$ are sub-boxes of $\Box_{k_1,k_2}$ such that $\Box_{k_1,k_2}=\Box'_{k_1,k_2}\cup\Box''_{k_1,k_2}$ and
    $\Box'_{k_1,k_2}\cap\Box''_{k_1,k_2}=\emptyset$. 
\end{enumerate} 

\begin{figure}[h]
\centering
\includegraphics[width = 0.35 \textwidth]{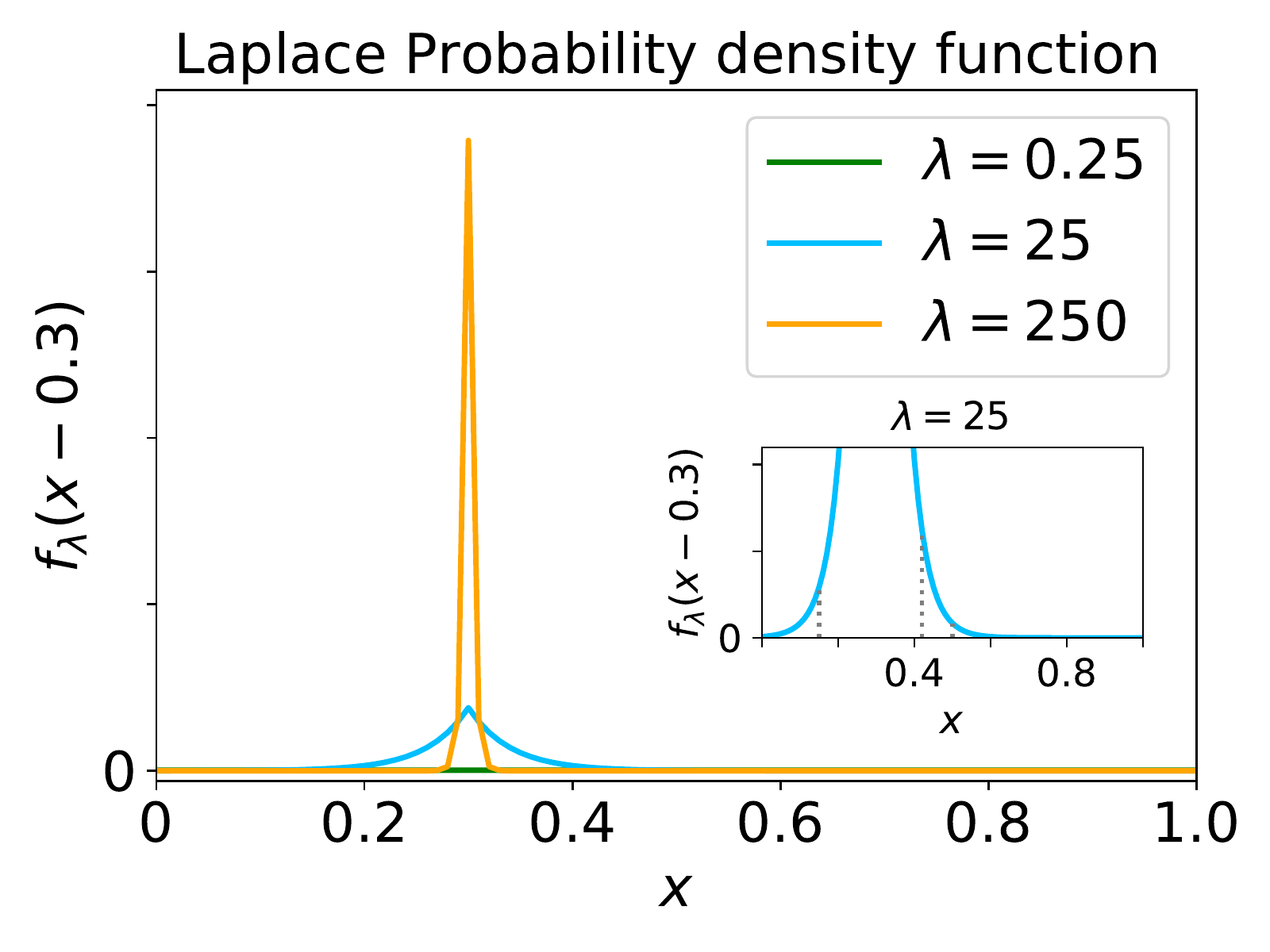}
\includegraphics[width = 0.36 \textwidth]{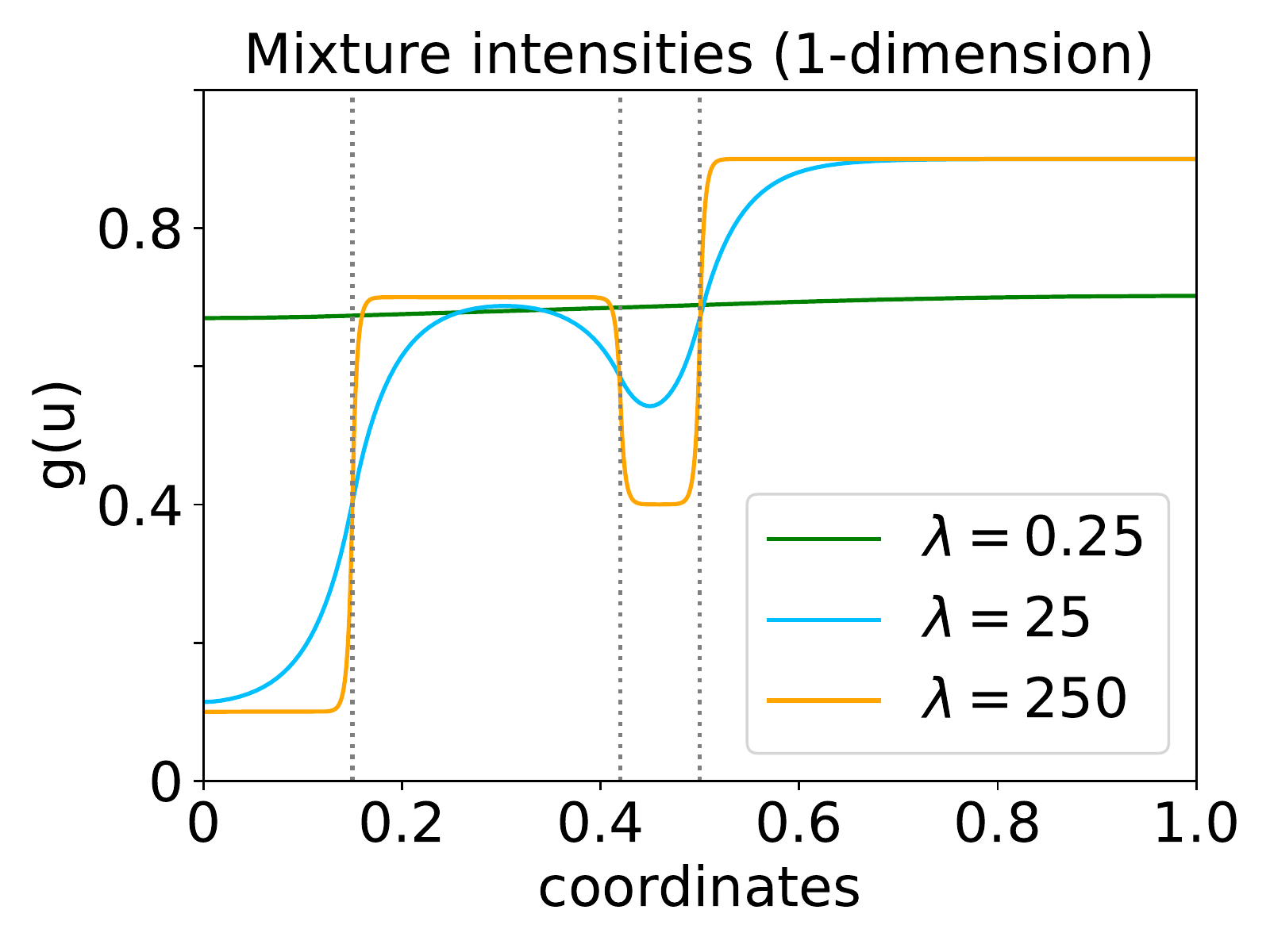}
\caption{The influence of the $\lambda$ parameter on the Laplace probability density function with coordinate located at $u=0.3$~(top), and the corresponding mixture intensities for $\{u_i\}_{i=1}^n$~(bottom). Different colors represent different settings of $\lambda$. The gray dotted lines represent segment division in one dimension with $\pmb{\theta}=(0.15, 0.27, 0.08, 0.5)^\top\sim\text{Dirichlet}(1, 1, 1,1)$. }
\label{fig:weight_function_vis}
\end{figure}

\begin{figure}[ht]
\centering
\includegraphics[width = 0.45 \textwidth]{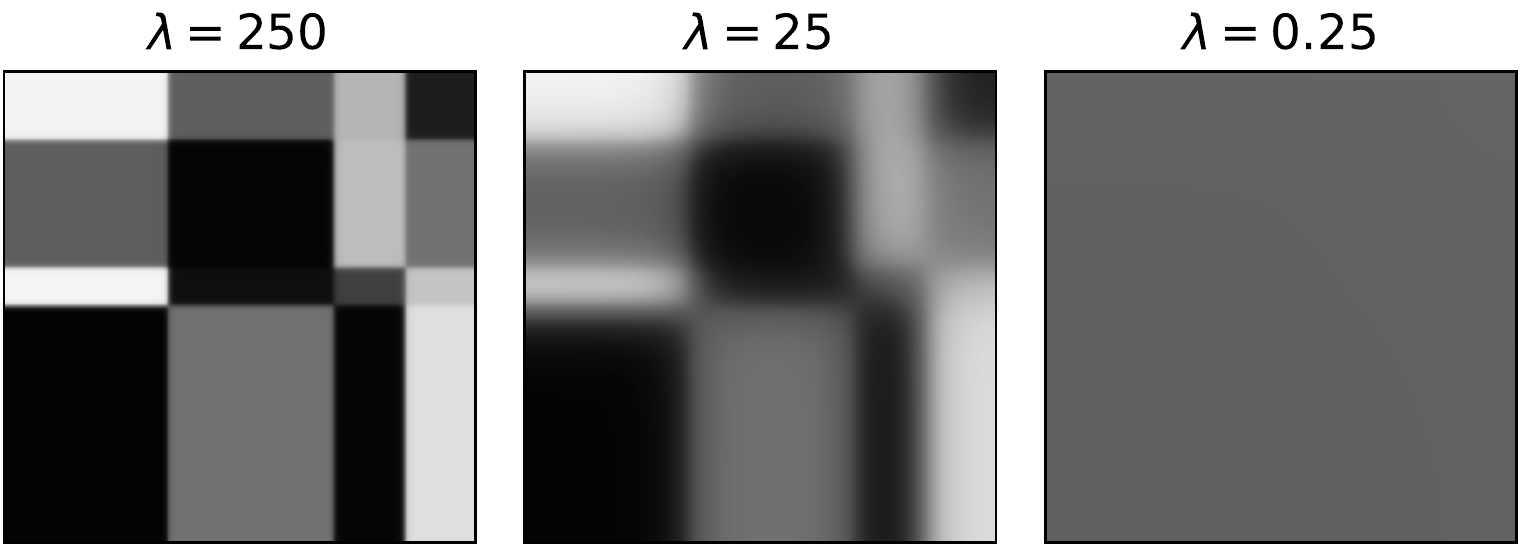}
\caption{Visualizations of the Integrated Smoothing Graphon under the Stochastic Block Model for different values of $\lambda$. Darker shading represents higher graphon intensity.}
\label{fig:lambda_related_visulizaton}
\end{figure}

There are many candidate functions satisfying these conditions, such as Gaussian or Laplace probability density functions. For ease of computation, we use the Laplace density (with location parameter $\mu = 0$) 
as the  derivative function, so that $f_{\lambda}(x-u)=\frac{{{\lambda}}}{2}e^{{-{{\lambda}} |x-u|}}$. Let $G_{\lambda}(x-u)=\left\{\begin{array}{lr}
    \frac{1}{2}e^{{{\lambda}} (x-u)}; & (x-u)<0 \\
    1-\frac{1}{2}e^{-{{\lambda}} (x-u)}; & (x-u)\ge0
\end{array}\right.$. We then have $\int_{L_{k_1-1}}^{L_{k_1}}f_{\lambda}(x-u)dx=G_{\lambda}(L_{k_1}-u)-G_{\lambda}(L_{k_1-1}-u) $. As a result, for relation $R_{ij}$ and corresponding node coordinates $(u^{(1)}_i, u^{(2)}_j)$, the normalised weight $\bar{F}_{\Box_{k_1,k_2}}(u^{(1)}_i, u^{(2)}_j)$ of the $(k_1, k_2)$-th block $\Box_{k_1,k_2}$ contributing to the mixture intensity of $R_{ij}$ is given by
\begin{align} \label{eq:pi_val}
    \bar{F}_{\Box_{k_1,k_2}}(u^{(1)}_i, u^{(2)}_j) =&  \frac{G_{\lambda}(L^{(1)}_{k_1}-u^{(1)}_i)-G_{\lambda}(L_{k_1-1}^{(1)}-u^{(1)}_i)}{G_{\lambda}(1-u^{(1)}_i)-G_{\lambda}(-u^{(1)}_i)}\nonumber \\
    &\times\frac{G_{\lambda}(L^{(2)}_{k_2}-u^{(2)}_j)-G_{\lambda}(L_{k_2-1}^{(2)}-u^{(2)}_j)}{G_{\lambda}(1-u^{(2)}_j)-G_{\lambda}(-u^{(2)}_j)}.
\end{align}

\begin{proposition}
$\sum_{k_1,k_2}\bar{F}_{\Box_{k_1,k_2}}(u^{(1)}_i, u^{(2)}_j)=1$.
\end{proposition}
\begin{proof}
{\small\begin{align}
&\sum_{k_1,k_2}\bar{F}_{\Box_{k_1,k_2}}(u^{(1)}_i, u^{(2)}_j) \nonumber \\
=&\left[\sum_{k_1}\frac{G_{\lambda}(L^{(1)}_{k_1}-u^{(1)}_i)-G_{\lambda}(L_{k_1-1}^{(1)}-u^{(1)}_i)}{G_{\lambda}(1-u^{(1)}_i)-G_{\lambda}(-u^{(1)}_i)}\right]\nonumber \\
&\cdot\left[\sum_{k_2}\frac{G_{\lambda}(L^{(2)}_{k_2}-u^{(2)}_j)-G_{\lambda}(L_{k_2-1}^{(2)}-u^{(2)}_j)}{G_{\lambda}(1-u^{(2)}_j)-G_{\lambda}(-u^{(2)}_j)}\right]=1.
\end{align}}
\end{proof}

Fig.~\ref{fig:weight_function_vis} (top) illustrates the function curves of $f_{\lambda}(x-u)$ for $u=0.3$ and Fig.~\ref{fig:weight_function_vis} (bottom) shows the resulting one-dimensional mixture intensities 
under varying scale  parameter  values $\lambda=0.25, 25$ and $250$. 
It is easily observed that when $\lambda$ is smaller, both the curves of the derivative function and the mixture intensity become flatter and smoother. Conversely, for larger $\lambda$, the mixture intensity values (generated for the coordinate $0.3$) become more discrete.  Fig.~\ref{fig:lambda_related_visulizaton} visualizes the mixture intensities obtained by applying the ISG to the SBM under the same three $\lambda$ values.
\begin{proposition}
$\lambda$ controls the smoothness of the graphon, with $\lambda\rightarrow\infty$ recovering the piecewise-constant graphon, and $\lambda\rightarrow 0$ resulting in a globally constant graphon.\end{proposition}
\begin{proof}
Using the L'hospital rule, when $\lambda\rightarrow 0$, we have
\begin{align}
&\lim_{\lambda\rightarrow 0} \frac{G_{\lambda}(L^{(1)}_{k_1}-u^{(1)}_i)-G_{\lambda}(L_{k_1-1}^{(1)}-u^{(1)}_i)}{G_{\lambda}(1-u^{(1)}_i)-G_{\lambda}(-u^{(1)}_i)}\nonumber \\
=&\frac{l_{k_1}^{(1)}-u_i^{(1)}-(l_{{k_1}-1}^{(1)}-u_i^{(1)})}{1-u_i^{(1)}+u_i^{(1)}} = l_{k_1}^{(1)}-l_{{k_1}-1}^{(1)}
\end{align}
Thus, we get $\bar{F}_{\Box_{k_1,k_2}}(u^{(1)}_i, u^{(2)}_j) = (l_{k_1}^{(1)}-l_{k_1-1}^{(1)})(l_{k_2}^{(2)}-l_{k_2-1}^{(2)})$, which is unrelated to the coordinate of $(u^{(1)}_i, u^{(2)}_j)$. The graphon is a globally constants graphon. 

We have three different cases when $\lambda\rightarrow \infty$: case (1), $L_{k_1}^{(1)}>L_{k_1-1}^{(1)}>u_i^{(1)}$, we have 
\begin{align}
&\lim_{\lambda\rightarrow \infty} \frac{G_{\lambda}(L^{(1)}_{k_1}-u^{(1)}_i)-G_{\lambda}(L_{k_1-1}^{(1)}-u^{(1)}_i)}{G_{\lambda}(1-u^{(1)}_i)-G_{\lambda}(-u^{(1)}_i)}\nonumber \\
= & \lim_{\lambda\rightarrow \infty} \frac{1-\frac{1}{2}e^{-\lambda(L^{(1)}_{k_1}-u^{(1)}_i)}-(1-\frac{1}{2}e^{-\lambda(L^{(1)}_{k_1-1}-u^{(1)}_i)})}{1-\frac{1}{2}e^{-\lambda(1-u^{(1)}_i)}-\frac{1}{2}e^{-\lambda(u^{(1)}_i)}}=0;\nonumber
\end{align}
case (2), $L_{k_1-1}^{(1)}<L_{k_1-1}^{(1)}<u_i^{(1)}$, we have 
\begin{align}
&\lim_{\lambda\rightarrow \infty} \frac{G_{\lambda}(L^{(1)}_{k_1}-u^{(1)}_i)-G_{\lambda}(L_{k_1-1}^{(1)}-u^{(1)}_i)}{G_{\lambda}(1-u^{(1)}_i)-G_{\lambda}(-u^{(1)}_i)}\nonumber \\
= & \lim_{\lambda\rightarrow \infty} \frac{\frac{1}{2}e^{\lambda(L^{(1)}_{k_1}-u^{(1)}_i)}-(\frac{1}{2}e^{\lambda(L^{(1)}_{k_1-1}-u^{(1)}_i)})}{1-\frac{1}{2}e^{-\lambda(1-u^{(1)}_i)}-\frac{1}{2}e^{-\lambda(u^{(1)}_i)}}=0;\nonumber
\end{align}
case (3), $L_{k_1-1}^{(1)}<u_i^{(1)}<L_{k_1-1}^{(1)}$, we have 
\begin{align}
&\lim_{\lambda\rightarrow \infty} \frac{G_{\lambda}(L^{(1)}_{k_1}-u^{(1)}_i)-G_{\lambda}(L_{k_1-1}^{(1)}-u^{(1)}_i)}{G_{\lambda}(1-u^{(1)}_i)-G_{\lambda}(-u^{(1)}_i)}\nonumber \\
= & \lim_{\lambda\rightarrow \infty} \frac{1-\frac{1}{2}e^{-\lambda(L^{(1)}_{k_1}-u^{(1)}_i)}-(\frac{1}{2}e^{\lambda(L^{(1)}_{k_1-1}-u^{(1)}_i)})}{1-\frac{1}{2}e^{-\lambda(1-u^{(1)}_i)}-\frac{1}{2}e^{-\lambda(u^{(1)}_i)}}=1;\nonumber
\end{align}
That is, $\bar{F}_{\Box_{k_1,k_2}}(u^{(1)}_i, u^{(2)}_j)=1$ if and only if the coordinate $(u^{(1)}_i, u^{(2)}_j)$ locates in the $(k_1, k_2)$-th block. Thus, the graphon would turn into a piecewise-constant one when $\lambda\rightarrow\infty$. 
\end{proof}
Accordingly, we refer to $\lambda$ as the smoothing parameter.

\subsection{Latent Feature Smoothing Graphon~(LFSG) with probabilistic assignment}
While the ISG addresses the limitations of the SBM-graphon by generating continuous intensity values, its graphon function \eqref{smooth_graphon_intensity_function} indicates that all blocks are involved in calculating the mixture intensity for generating individual relations. Accordingly, the additive form for evaluating the mixture intensity makes it difficult to form efficient inference schemes for all  random variables. 
To
improve 
 inferential
efficiency we introduce auxiliary pairwise latent labels $\{s_{ij}\}_{j=1}^n$ (associated with node $i$) and $\{r_{ij}\}_{i=1}^n$ (associated with node $j$) for individual relations $\{R_{ij}\}_{i,j=1}^n$, where $s_{ij},r_{ij} \in\{1, \ldots, K\}$. The $\{s_{ij}\}_{j=1}^n$ and $\{r_{ij}\}_{i=1}^n$ are sampled from the respective node categorical distributions in their corresponding dimensions using normalised weights as probabilities. In particular
\begin{align} \label{eq:s_ij_r_ij}
\{s_{ij}\}_{j=1}^n&\sim\text{Categorical}(\bar{F}^{(1)}_{1}(u^{(1)}_i), \ldots, \bar{F}^{(1)}_{K}(u^{(1)}_i)) \nonumber \\
\{r_{ij}\}_{i=1}^n&\sim\text{Categorical}(\bar{F}^{(2)}_{1}(u^{(2)}_j), \ldots, \bar{F}^{(2)}_{K}(u^{(2)}_j)),
\end{align} 
where $\bar{F}_{k}(u)=\frac{G_{\lambda}(L_{k}-u)-G_{\lambda}(L_{k-1}-u)}{G_{\lambda}(1-u)-G_{\lambda}(-u)}$ is the normalised weight of segment $k$ in the dimension of coordinate $u$. 
For each relation from node $i$ to node $j$ ($R_{ij}$), the hidden label $s_{ij}$ denotes the group that node $i$ belongs to (in the $1$st dimension) and $r_{ij}$ denotes the group that node $j$ belongs to (in the $2$nd dimension). Through the introduction of the two labels, the final intensity in determining $R_{ij}$ can be obtained similarly to the Mixed Membership Stochastic Block Model (MMSB)~\cite{airoldi2009mixed}:
\begin{align}
    P(R_{ij}=1|s_{ij}, r_{ij}, \pmb B)=B_{s_{ij}, r_{ij}}.
\end{align}
Note that since both $\{s_{ij}\}_{j=1}^n$ and $\{r_{ij}\}_{j=1}^n$ are $n$-element arrays, each node has the potential to belong to multiple segments, rather than the single segment under the SBM-graphon. When participating in different relations, each outgoing node $i$ (incoming node $j$) may fall into different segments, which means that each node may play different roles when taking part in different relations.
Note that taking expectations over the hidden labels $s_{ij}$ and $r_{ij}$, results in the same intensity as for the ISG, so that
\begin{align}
    \mathbb{E}_{s_{ij}, r_{ij}}\left[P(R_{ij}=1\vert s_{ij}, r_{ij}, \pmb B)\right]= g\left((u_i^{(1)}, u_j^{(2)})\right).
\end{align}
We term this approach the Latent Feature Smoothing Graphon~(LFSG). Its generative process is described as follows:

1)$\sim$3) The block intensities~($\pmb B$), graphon partition~($\boxplus$) and $2$-dimensional coordinates ($\{u_{i}^{(1)}, u_{i}^{(2)}\}_{i=1}^n$) are generated as for the SBM-graphon;
\begin{enumerate}
  \setcounter{enumi}{3}
    \item For $i=1, \cdots, n$, calculate the hidden label distributions in each dimension, $\bar{\pmb{F}}^{(1)}(u_i^{(1)})$ and $\bar{\pmb{F}}^{(2)}(u_i^{(2)})$, where $\bar{\pmb{F}}^{(1)}(u_i^{(1)})=({\bar{F}}^{(1)}_1(u_i^{(1)}), \ldots, {\bar{F}}^{(1)}_K(u_i^{(1)}))$; 
    \item For $i, j=1, \cdots, n$, 
    \begin{enumerate}
        \item Generate the hidden labels $s_{ij}\sim\bar{\pmb{F}}^{(1)}(u_i^{(1)}), r_{ij}\sim\bar{\pmb{F}}^{(2)}(u_j^{(2)})$ following \eqref{eq:s_ij_r_ij}
        \item Generate $R_{ij}\sim\text{Bernoulli}\left(B_{s_{ij}, r_{ij}}\right)$.
    \end{enumerate}
\end{enumerate}

The graphical models for implementing the ISG within the SBM (referred to as the ISG-SBM), and also implementing the LFSG within the SBM (referred to as the LFSG-SBM)  are illustrated in Fig.~\ref{fig:graphical_model}.
The main difference between the two models  -- the introduction of the pairwise hidden labels $s_{ij}$ and $r_{ij}$ for generating each relation $R_{ij}$ -- allows the LFSG-SBM to enjoy the following advantages over the ISG-SBM:
\begin{itemize}
\item The aggregated counting information of the hidden labels enables efficient Gibbs sampling of the block intensities $\pmb{B}$.
\item Calculation involving $\bar{F}_{\Box}$ is instead reduced to calculation involving $\bar{F}_{k}(u)$, avoiding  the inclusion of all blocks when calculating the mixture intensity.
\item 
Because each node is allowed to belong to different groups when linking to other nodes, permitting differences in the  natures of those links, the group distribution $\bar{F}(u)$ is then easily interpretable as  the group membership distribution for that node.
For example, a higher membership degree in  group $k$ indicates the node is more important or active in group $k$.
\end{itemize}

\begin{figure}[h]
\centering
\includegraphics[width = 0.45 \textwidth]{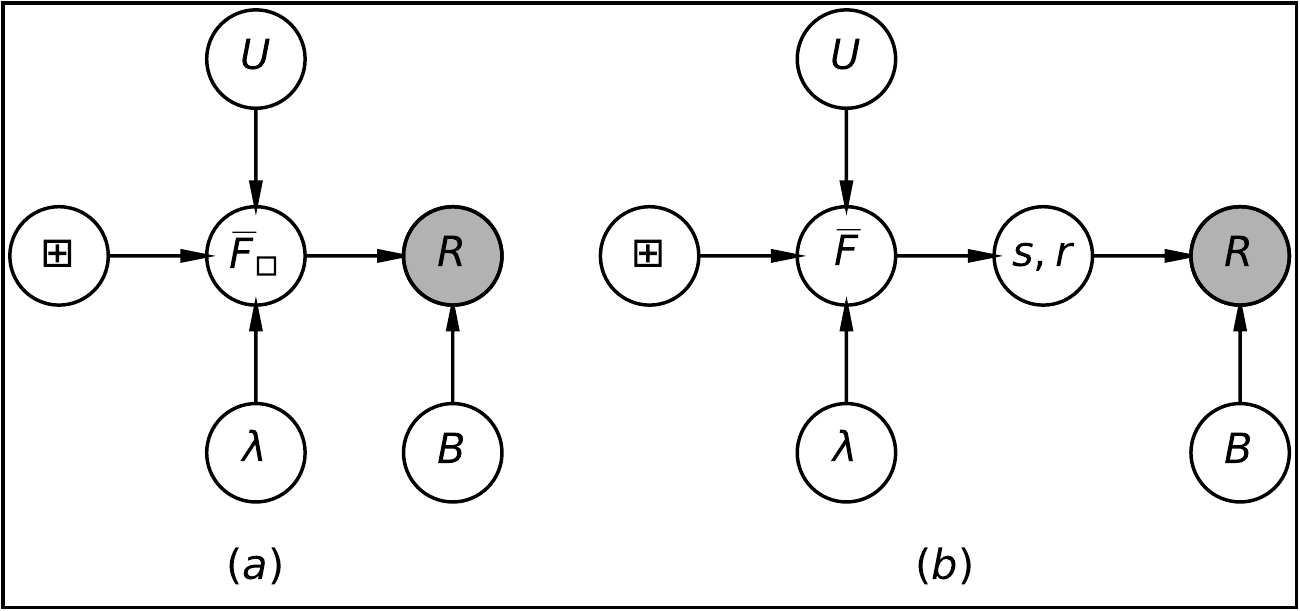}
\caption{The graphical model for (a) the ISG-SBM and (b) the LFSG-SBM.  (a) The weights of the blocks $\bar{F}_{\Box}$ are first calculated by using the partition $\boxplus$, node coordinates $U$ and smoothing parameter $\lambda$. Then the $\bar{F}_{\Box}$ and block intensities $\pmb B$ are integrated together to generate the exchangeable relations $R$. (b) The weight $\bar{F}$ for each node is individually generated using the partition $\boxplus$, node coordinates $U$ and the smoothing parameter $\lambda$, based on the auxiliary hidden labels $s, r$ for the node pair relationship. Then $s, r$ are used to generate the exchangeable relational, $R$, together with the block intensities $\pmb B$.}
\label{fig:graphical_model}
\end{figure}

\subsection{Extensions to other piecewise-constant graphons}
The major difference between the construction of the existing piecewise-constant graphons is the generation process of partitions~($\boxplus$; Fig.~\ref{fig:graphon_comparison}). As a result, our smoothing approach, while described for the  SBM-graphon, can be straighforwardly applied to other piecewise-constant graphons. For example, to apply the ISG to other piecewise-constant graphons, we can similarly calculate a mixture intensity as a weighted sum of the intensities of all existing blocks.
When the partitioned blocks are rectangular-shaped (as for e.g.~the MP-graphon, RTP-graphon and RBP-graphon), the intensity for each block can be computed by independently integrating the derivative function over two dimensions. If the partitioned blocks are shaped as convex-polygons (as for e.g.~the BSP-graphon~\cite{pmlr-v84-fan18b}), the intensity can be generated via integrating the derivative function over the polygon.

\section{Inference}

\begin{algorithm}[t]
\caption{{MCMC for the ISG}}
\label{detailInference_ISG}
\begin{algorithmic}
  \REQUIRE Exchangeable relational data $R\in\{0, 1\}^{n\times n}$, hyperparameters $\alpha_0, \beta_0, \pmb{\alpha}_{1\times K}$, iteration time $T$
  \ENSURE $\{u_{i}^{(1)}, u_{i}^{(2)}\}_{{i=1}}^n, \pmb{\theta}^{(1)}, \pmb{\theta}^{(2)}, \pmb{B}, \lambda$
  \FOR{$t=1,\cdots, T$}
    \FOR{$i=1,\ldots,n$}
    \STATE Sample $u_i^{(1)}, u_i^{(2)}$; // according to \eqref{eq:sample_u_i}   
    \ENDFOR
      \STATE Sample $\pmb{\theta}^{(1)}, \pmb{\theta}^{(2)}$; // according to \eqref{eq:sample_theta}
    \FOR{$k_1, k_2=1,\ldots,K$}
      \STATE Sample $B_{k_1,k_2}$; // according to \eqref{eq:sample_B_k1k2}
    \ENDFOR
      \STATE Sample $\lambda$; // according to \eqref{eq:sample_lambda}
  \ENDFOR
\end{algorithmic}
\end{algorithm}

\begin{algorithm}[t]
\caption{MCMC for the LFSG}
\label{detailInference_SG}
\begin{algorithmic}
  \REQUIRE Exchangeable relational data $R\in\{0, 1\}^{n\times n}$, hyperparameters $\alpha_0, \beta_0, \pmb{\alpha}_{1\times K}$, iteration time $T$
  \ENSURE $\{u_{i}^{(1)}, u_{i}^{(2)}\}_{{i=1}}^n, \pmb{\theta}^{(1)}, \pmb{\theta}^{(2)}, \pmb{B}, \{s_{ij},r_{ij}\}_{i,j=1}^n, \lambda$
  \FOR{$t=1,\cdots, T$}
    \FOR{$i=1,\ldots,n$}
    \STATE Sample $u_i^{(1)}, u_i^{(2)}$; // according to \eqref{eq:sample_u_i}
    \ENDFOR
      \STATE Sample $\pmb{\theta}^{(1)}, \pmb{\theta}^{(2)}$; // according to \eqref{eq:sample_theta}
    \FOR{$k_1, k_2=1,\ldots,K$}
      \STATE Sample $B_{k_1,k_2}$; // according to \eqref{eq:sample_B_k1k2}
    \ENDFOR
    \FOR{$i,j=1,\ldots,n$}
    \STATE Sample $s_{ij},r_{ij}$; // according to \eqref{eq:sample_s_r_ij}
    \ENDFOR
      \STATE Sample $\lambda$; // according to \eqref{eq:sample_lambda}
  \ENDFOR
\end{algorithmic}
\end{algorithm}

We present a Markov Chain Monte Carlo~(MCMC) algorithm for posterior model inference, with detailed steps for the ISG and the LFSG  as illustrated in Algorithms \ref{detailInference_ISG} and \ref{detailInference_SG} respectively.
In general, the joint distribution over the hidden labels $\{s_{ij}, r_{ij}\}_{i,j=1}^n$,  pairwise node coordinates $\{u_{i}^{(1)}, u_{i}^{(2)}\}_{i=1}^n$, group distributions $\pmb{\theta}^{(1)}, \pmb{\theta}^{(2)}$, the block intensities $\pmb{B}$ and the smoothing parameter $\lambda$ is:
\begin{align}
   & P(\{s_{ij}, r_{ij}, R_{ij}\}_{i,j=1}^n, \{u_{i}^{(1)}, u_{i}^{(2)}\}_{{i=1}}^n, \pmb{\theta}^{(1)}, \pmb{\theta}^{(2)}, \pmb{B}, \lambda|\alpha_0, \beta_0)\nonumber \\
   \propto & \prod_{i,k}\left[\bar{F}_{k}^{(1)}(u_i^{(1)}|\pmb{\theta}^{(1)}, \lambda)^{m_{ik}^{(1)}}\bar{F}_{k}^{(2)}(u_i^{(2)}|\pmb{\theta}^{(2)}, \lambda)^{m_{ik}^{(2)}}\right]\nonumber \\
   &\cdot \prod_{k_1, k_2}\left[B_{k_1, k_2}^{N_{k_1, k_2}^{(1)}+\alpha_0-1}(1-B_{k_1, k_2})^{N_{k_1, k_2}^{(0)}+\beta_0-1}\right]\nonumber \\
   &\cdot\prod_{k}\left[(\theta_{k}^{(1)})^{\alpha_{k}-1}(\theta_{k}^{(2)})^{\alpha_{k}-1}\right]\cdot P(\lambda)
\end{align}
where $m_{ik}^{(1)}=\sum_{j=1}^n\pmb{1}(s_{ij}=k), m_{{ik}}^{(2)}=\sum_{j=1}^n\pmb{1}(r_{ji}=k), N_{k_1, k_2}^{(1)}=\sum_{(i,j):s_{ij}=k_1, r_{ij}=k_2}\pmb{1}(R_{ij}=1), N_{k_1, k_2}^{(0)}=\sum_{(i,j):s_{ij}=k_1, r_{ij}=k_2}\pmb{1}(R_{ij}=0)$. In this joint distribution, we have set the following prior distributions for the variables: $s_{ij}\sim\text{Categorical}(\bar{\pmb{F}}(u_i^{(1)}|\pmb{\theta}^{(1)}, \lambda)), B_{k_1,k_2}\sim\text{Beta}(\alpha_0, \beta_0), \pmb{\theta}^{(1)}\sim\text{Dirichlet}(\pmb{\alpha}_{1\times K})$.

The details for updating each parameter in the ISG and LFSG MCMC algorithms are listed below. 

\paragraph{Updating $\{u_{i}^{(1)}, u_{i}^{(2)}\}_{i=1}^n$}
Independent Metropolis-Hastings steps can be used to update the variables $u_{i}^{(1)}, u_{i}^{(2)}$. We propose a new sample for $u_{i}^{(1)}$ 
as $u^*\sim\text{Beta}[\alpha_u, \beta_u]$, and accept this proposal with probability $\min(1, \alpha_{u_i^{(1)}})$ where
\begin{align} \label{eq:sample_u_i}
\alpha_{u_i^{(1)}} = \frac{Be(u_i^{(1)}|\alpha_u,\beta_u)}{Be(u^*|\alpha_u,\beta_u)}
\prod_{k}\frac{\bar{F}_{k}^{(1)}(u^{*}|\pmb{\theta}^{(1)}, \lambda)^{m_{ik}^{(1)}}}{\bar{F}^{(1)}_{k}(u_i^{(1)}|\pmb{\theta}^{(1)}, \lambda)^{m_{ik}^{(1)}}},
\end{align}
where $Be(u|\alpha,\beta)$ denotes the Beta density with parameters $\alpha$ and $\beta$ evaluated at $u$. 
The update for $u_i^{(2)}$  proceeds similarly. Note that each of the $2n$ parameters $\{u_i^{(1)},u_i^{(2)}\}_{i=1}^n$ can be updated in parallel. In our simulations we found that $\alpha_u=\beta_u=1$ gave good sampler performance.

\paragraph{Updating $\pmb{\theta}^{(1)}, \pmb{\theta}^{(2)}$}
A Metropolis-Hastings step can be used to update  $\pmb{\theta}^{(1)}$, and $\pmb{\theta}^{(2)}$. For $\pmb{\theta}^{(1)}$ or $\pmb{\theta}^{(2)}$ we draw a proposed sample  $\pmb{\theta}^*\sim\text{Dirichlet}(\pmb{\alpha}_{1\times K})$ from a Dirichlet distribution with concentration parameters $\pmb{\alpha}_{1\times K}$. We accept the proposal $\pmb{\theta}^*$ for w.l.o.g.~$\pmb{\theta}^{(1)}$ with probability $\min(1, \alpha_{\pmb{\theta}^{(1)}})$, where
\begin{align} \label{eq:sample_theta}
\alpha_{\pmb{\theta}^{(1)}} = \prod_{i,k}\frac{\bar{F}^{(1)}_{k}(u_i^{(1)}|\pmb{\theta}^{*}, \lambda)^{m_{ik}^{(1)}}}{\bar{F}^{(1)}_{k}(u_i^{(1)}|\pmb{\theta}^{(1)}, \lambda)^{m_{ik}^{(1)}}},
\end{align}
with a similar update for $\pmb{\theta}^{(2)}$. Both $\pmb{\theta}^{(1)}$ and $\pmb{\theta}^{(2)}$ can be updated in parallel.

\paragraph{Updating $\pmb B$}
The conjugacy between the prior  and the conditional likelihood for $\pmb B$ means that we can update $\pmb B$ via a Gibbs sampling step.
Specifically, each entry $B_{k_1, k_2}$ can be updated in parallel via
\begin{align} \label{eq:sample_B_k1k2}
B_{k_1, k_2}\sim\text{Beta}(\alpha_0+N_{k_1, k_2}^{(1)}, \beta_0+N_{k_1, k_2}^{(0)}), \forall k_1, k_2.
\end{align}

\paragraph{Updating $\{s_{ij}, r_{ij}\}_{i,j=1}^n$}
The posterior distribution of $s_{ij}$ is a categorical distribution, where the probability of $s_{ij}=k$ is
\begin{align} \label{eq:sample_s_r_ij}
   P(s_{ij}=k|\theta_{k}^{(i)}, R_{ij}, B_{k,r_{ij}})\propto\: & \bar{F}_{k}^{(1)}(u_i^{(1)}|\pmb{\theta}^{(1)}, \lambda)\\&\times B_{k,r_{ij}}^{R_{ij}}\nonumber 
    (1-B_{k,r_{ij}})^{1-R_{ij}},
\end{align}
and from which $s_{ij}$ may be straightforwardly updated ($r_{ij}$ may be updated in a similar way).
Each of the $2n$ parameters can be updated in parallel.

\paragraph{Updating $\lambda$}
A Metropolis-Hastings step can be used to update $\lambda$. We draw a proposed value $\lambda^*\sim\Gamma(1, 1)$ and accept it with probability $\min(1, \alpha_{\lambda})$, where
\begin{align} \label{eq:sample_lambda}
\alpha_{\lambda}=\prod_{i,k}\frac{\bar{F}^{(1)}_{k}(u_i^{(1)}|\pmb{\theta}^{(1)}, \lambda^*)^{m_{i{k}}^{(1)}}\bar{F}^{(2)}_{k}(u_i^{(2)}|\pmb{\theta}^{(2)}, \lambda^*)^{m_{i{k}}^{(2)}}}{\bar{F}^{(1)}_{k}(u_i^{(1)}|\pmb{\theta}^{(1)}, \lambda)^{m_{ik}^{(1)}}\bar{F}^{(2)}_{k}(u_i^{(2)}|\pmb{\theta}^{(2)}, \lambda)^{m_{ik}^{(2)}}}.
\end{align} 




\section{Related work}
There are many notable Bayesian methods for modelling exchangeable relational data. First, we review related models that can be viewed via graphon theory. We then discuss the Mixed-Membership Stochastic Block Model~(MMSB)~\cite{airoldi2009mixed} and highlight the differences between the MMSB and the LFSG. Finally, we analyse and compare the computational complexities of our model compared to existing methods.

\subsection{Graphons for modelling exchangeable relational data}
The Mondrian process relational model~(MP-RM; Fig.~\ref{fig:graphon_comparison}, centre-left)~\cite{roy2007learning,roy2009mondrian,roy2011thesis} is a representative model which generates $k$-d tree structured piecewise-constant graphons. In general, the Mondrian process recursively generates axis-aligned cuts in the unit square 
and partitions the space in a hierarchical fashion known as a $k$-d tree. The tree structure is regulated by attaching an exponentially distributed cost to each axis-aligned cut, so that the tree generation process terminates when the accumulated cost exceeds a budget value. 
The Binary Space Partitioning-Tree process relational model~(BSP-RM)~\cite{pmlr-v84-fan18b,pmlr-v89-fan18a} also generates tree structured partitions. The difference between the BSP-RM and the MP-RM is that the BSP-RM uses two dimensions to form oblique cuts and thus generate convex polyhredon-shaped blocks. These oblique cuts concentrate more on describing the inter-dimensional dependency and can produce more efficient space partitions.

The regular-grid piecewise-constant graphon is similar earlier model (Fig.~\ref{fig:graphon_comparison}, left). Generally, it is constructed from $2$ independent partition processes in a $2$-dimensional space. The resulting orthogonal crossover between these dimensions produces regular grids in the space. Typical regular-grid partition models include the SBM~\cite{nowicki2001estimation} and its infinite states variant, the Infinite Relational Model~(IRM)~\cite{kemp2006learning}. The SBM uses a Dirichlet distribution~(or Dirichlet process for the IRM) to independently generate a finite~(or infinite for the IRM) number of segments in each dimension.

The Rectangular Tiling process relational model~(RTP-RM; Fig.~\ref{fig:graphon_comparison}, centre-right)~\cite{nakano2014rectangular} produces a flat partition structure on a two-dimensional array by assigning each entry to an existing block or a new block in sequence, without violating the rectangular restriction of the blocks. By relaxing the restrictions of the hierarchical or regular-grid structure, the RTP-RM aims to provide more flexibility in block generation. However, the process of generating blocks is quite complicated for practical usage. Moreover,  the hierarchical and regular-grid partition models can be used for continuous space and multi-dimensional arrays~(after trivial modifications), the RTP-RM is restricted to (discrete) arrays only. 

The Rectangular Bounding process relational model~(RBP-RM)~\cite{NIPS2018_RBP} uses a bounding strategy to generate rectangular blocks in the space. In contrast to the previously described cutting strategies, the RBP-RM concentrates more on the important regions of the space and avoids over-modelling sparse and noisy regions. In the RBP-RM, the number of possible intensities 
is equivalent to the number of blocks, which follows a Poisson distribution and is finite almost surely.

The Gaussian process relational model~(GP-RM; Fig.~\ref{fig:graphon_comparison}, right)~\cite{RandomFunPriorsExchArrays} utilises a prior over a random function in the unit square to form a continuous graphon. In this way it can potentially generate desired continuous intensity values via the graphon function. However, the computational cost of the GP-RM is the same as that of the Gaussian process, which  scales to the cube of the number of nodes ($n$). 

\subsection{Comparing the LFSG and the MMSB} The MMSB is another notable Bayesian method for modelling exchangeable relational data. In contrast to other graphon methods,
the MMSB allows each node $i$ to have a group distribution $\pmb{F}_i$, which follows a Dirichlet distribution. To form the relation between any two nodes $i,j$, a latent label pair consisting of a sender and a receiver $(s_{ij}, r_{ij})$ is first generated via $s_{ij}\sim\text{Categorical}(\pmb{F}_i)$, and $r_{ij}\sim\text{Categorical}(\pmb{F}_j)$. The relation $R_{ij}$ may then be generated based on the intensity of the block $\pmb{B}$ formed by group $s_{ij}$ and group $r_{ij}$: $R_{ij}\sim\text{Bernoulli}(B_{s_{ij},r_{ij}})$. 
Our proposed LFSG shares similarities with the MMSB, since both of them use group distributions to represent individual nodes and the likelihood generation method is the same. However, there are key differences. These are: (1) The priors for the group distributions are different. In the MMSB, the group distributions of all nodes are generated independently from a Dirichlet distribution, whereas  in the LFSG, the group distributions are highly correlated~(determined by the node coordinates and the unified partition structure); (2) The MMSB requires $nK$ parameters to form the group distributions, while the LFSG requires only {$2(n+K)$} parameters; (3) The MMSB cannot be described by graphon theory (because it involves $n$ independent group distribution $\pmb{F}_i$ for generating group distributions), whereas the LFSG naturally fits within the graphon framework. 

\begin{table}
\centering 
\caption{Model complexity comparison}
\begin{tabular}{l|cc}
  \hline
  Model & Intensity computation & Label sampling\\
  \hline
  SBM & $\mathcal{O}(K^2L)$ & $\mathcal{O}(nK)$     \\
  ISG & $\mathcal{O}(K^2n^2)$ & -- \\ 
  LFSG & $\mathcal{O}(K^2L)$ & $\mathcal{O}(n^2K)$  \\
  MMSB & $\mathcal{O}(K^2L)$ & $\mathcal{O}(n^2K)$  \\
  GP-RM & $\mathcal{O}(n^3)$ & -- \\ 
  \hline
\end{tabular}
\label{table:model_complexity}
\end{table}

\subsection{Computational complexities}
Table~\ref{table:model_complexity} compares the computational complexities of the ISG and the LFSG against representative existing models, including the SBM, the MMSB and the GP-RM. In terms of calculating the intensity for generating the relations $\{R_{ij}\}_{i,j=1}^n$, the ISG requires a scale of $\mathcal{O}(n^2K^2)$ since the calculation of  the mixture intensity for each relation involves a pair of coordinates~(giving a total of $n^2$) and all of the block intensities~(which is  $K^2$). However,  the uncoupling strategy applied in the LFSG lowers this cost dramatically to $\mathcal{O}(K^2L)$, where $L$ is the number of positive links~(i.e.~$R_{ij}=1$) observed in the data (Table \ref{rm_dataset} enumerates $L$ for each data set analysed below). Note that the mixture intensity computation cost of the LFSG is the same as 
\begin{figure*}[t]
\centering
\includegraphics[width = 0.9 \textwidth]{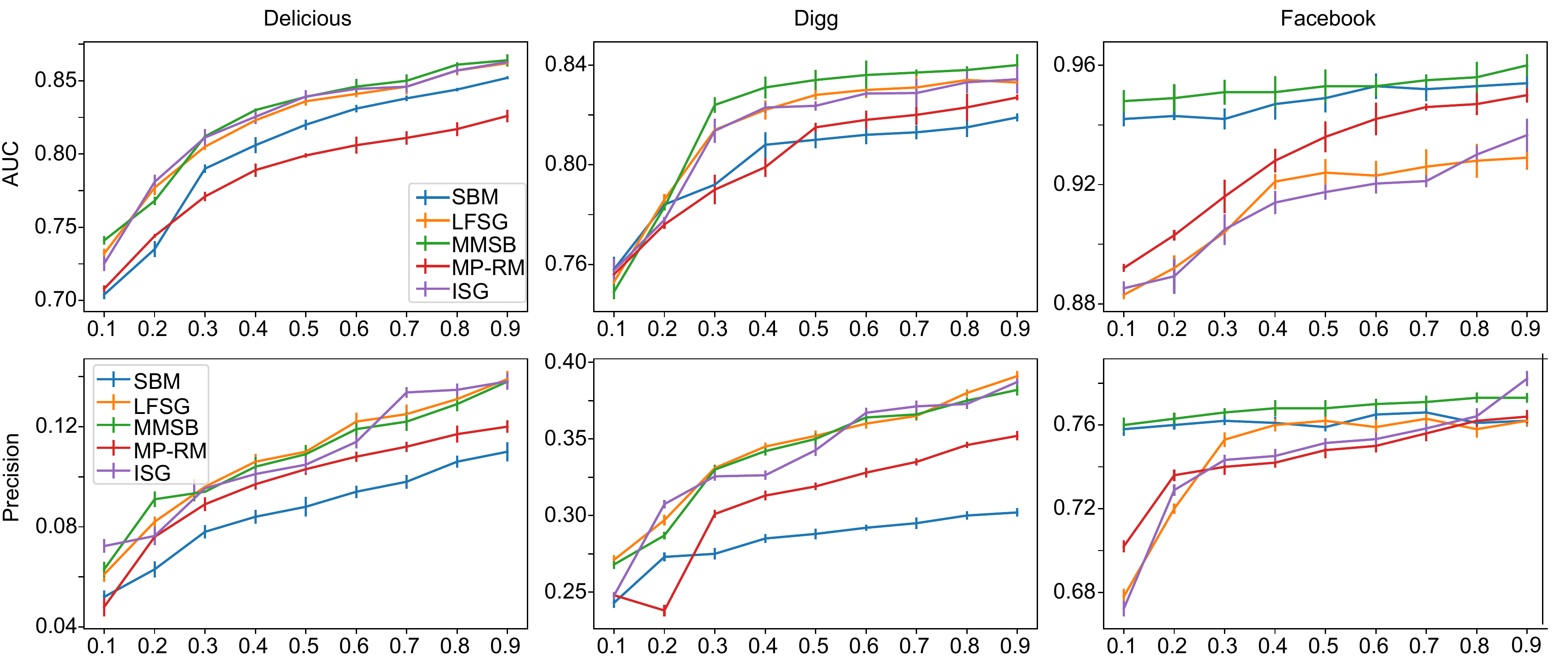}
\includegraphics[width = 0.9 \textwidth]{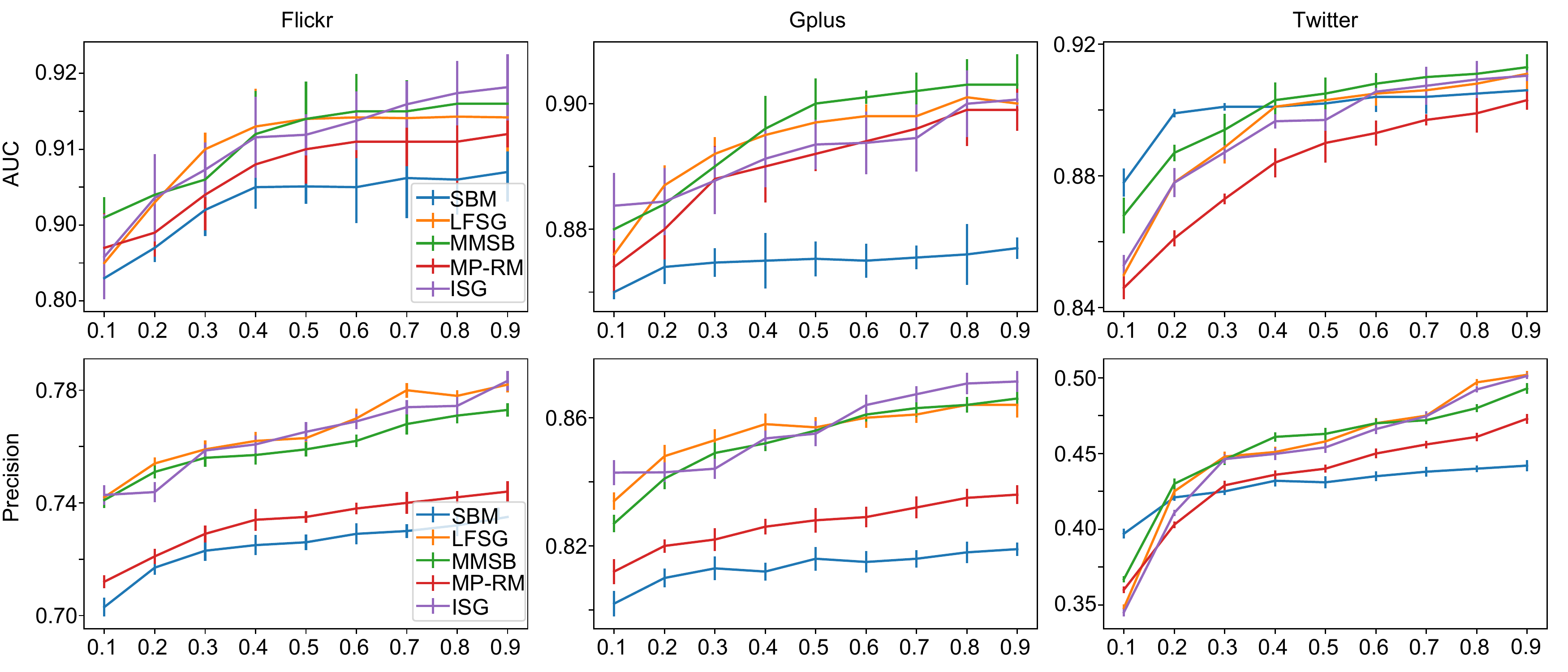}
\caption{Average area under the curve receiver operating characteristic~(AUC) and the precision recall~(Precision) under the Stochastic Block Model~(SBM, blue line), Mixed-membership Stochastic Block Model~(MMSB, green line), Mondrian Process-Relational Model~(MP-RM, red line), Latent Feature Smoothing Graphon on the SBM~(LFSG, orange line) and Integrated Smoothing Graphon on the SBM~(ISG, purple line) for each of the Delicious, Digg, Facebook, Flickr, Gplus and Twitter datasets, under different proportions of training data ($x$-axis).
}
\label{fig:auc_precision}
\end{figure*}
that of both the SBM and the MMSB. As a result, the continuous intensities of the LFSG compared to the discrete intensities of the SBM is achieved without sacrificing computation complexity. In contrast, the computational cost of computing the mixture intensity for the GP-RM is $\mathcal{O}(n^3)$, which is the highest among these methods, even though it can also provide continuous intensities. Regarding the complexity of sampling the labels, both the LFSG and the MMSB provide multiple labels for each node and incur the same cost of $\mathcal{O}(n^2K)$. However, while the SBM requires a smaller cost of $\mathcal{O}(nK)$ for label sampling, it only allows a single-label for each node. 

\section{Experiments}
We now evaluate the performance of the ISG-SBM and the LFSG-SBM on real-world data sets, comparing them with  three state-of-the-art methods: the SBM, the MP-RM and  the MMSB. Although the MMSB cannot be explained using graphon theory, it is included as it shares some similarities with with the LFSG-SBM. We implement posterior simulation for the SBM and the MMSB using Gibbs sampling  and  a conditional Sequential Monte Carlo algorithm\cite{andrieu2010particle,LakOryTeh2015ParticleGibbs,pmlr-v89-fan18a} for the MP-RM.

\begin{table}
\centering 
\caption{Dataset summary information}
\begin{tabular}{c|cc||c|cc}
  \hline
  Dataset & $L$ & $S(\%)$ & Dataset & $L$ & $S(\%)$\\
  \hline
  Delicious & $10,775$ & \phantom{1}$4.31 $ &   Gplus & $76,575$ & $30.63 $     \\
  Digg & $25,943$ & $ 10.38$ & Facebook & $54,476$ & $ 21.79$   \\
  Flickr & $49,524$ & $ 19.81$   &   Twitter & $24,378$ & $ \phantom{1}9.75$   \\
 \hline
\end{tabular}
\label{rm_dataset}
\end{table}
\subsection{Data sets} 

We examine six real-world exchangeable relational data sets: Delicious~\cite{Zafarani+Liu:2009}, Digg~\cite{Zafarani+Liu:2009}, Flickr~\cite{Zafarani+Liu:2009}, Gplus~\cite{leskovec2012learning}, Facebook~\cite{leskovec2012learning}, and Twitter~\cite{leskovec2012learning}. To construct the exchangeable relational data matrix we extract 
the top $1\,000$ active nodes based on  node interaction frequencies,
and then randomly sample $500$ nodes from these top $1\,000$ nodes to form the $500\times 500$ interaction binary matrix. Table~\ref{rm_dataset} summarizes the number of positive links~($L$) and the corresponding sparsity ($S\%$), which is defined as the ratio of the number of positive links to the total number of links, for each dataset.

\begin{figure*}[p]
\centering
\includegraphics[width = 0.49 \textwidth]{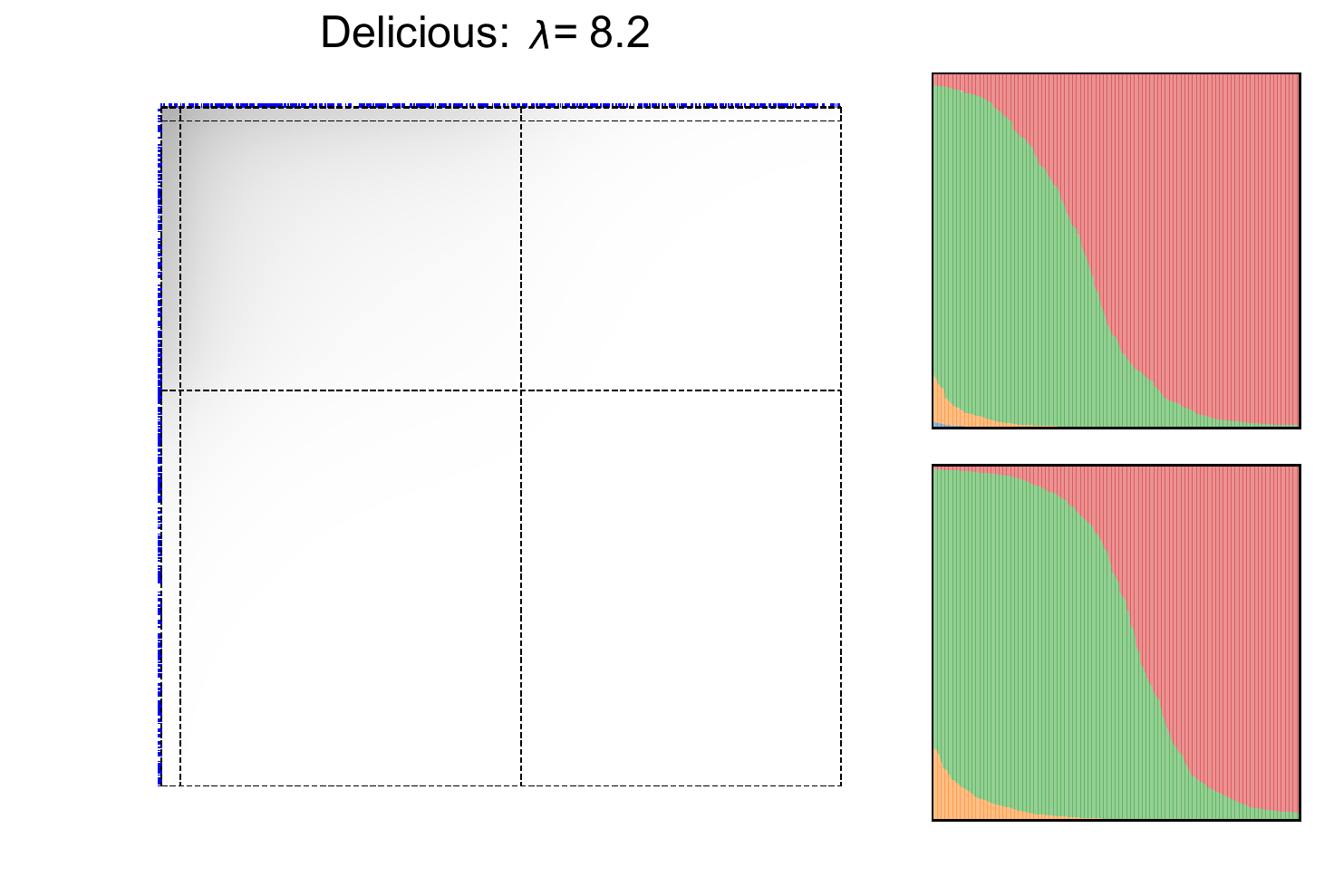}\quad
\includegraphics[width = 0.49 \textwidth]{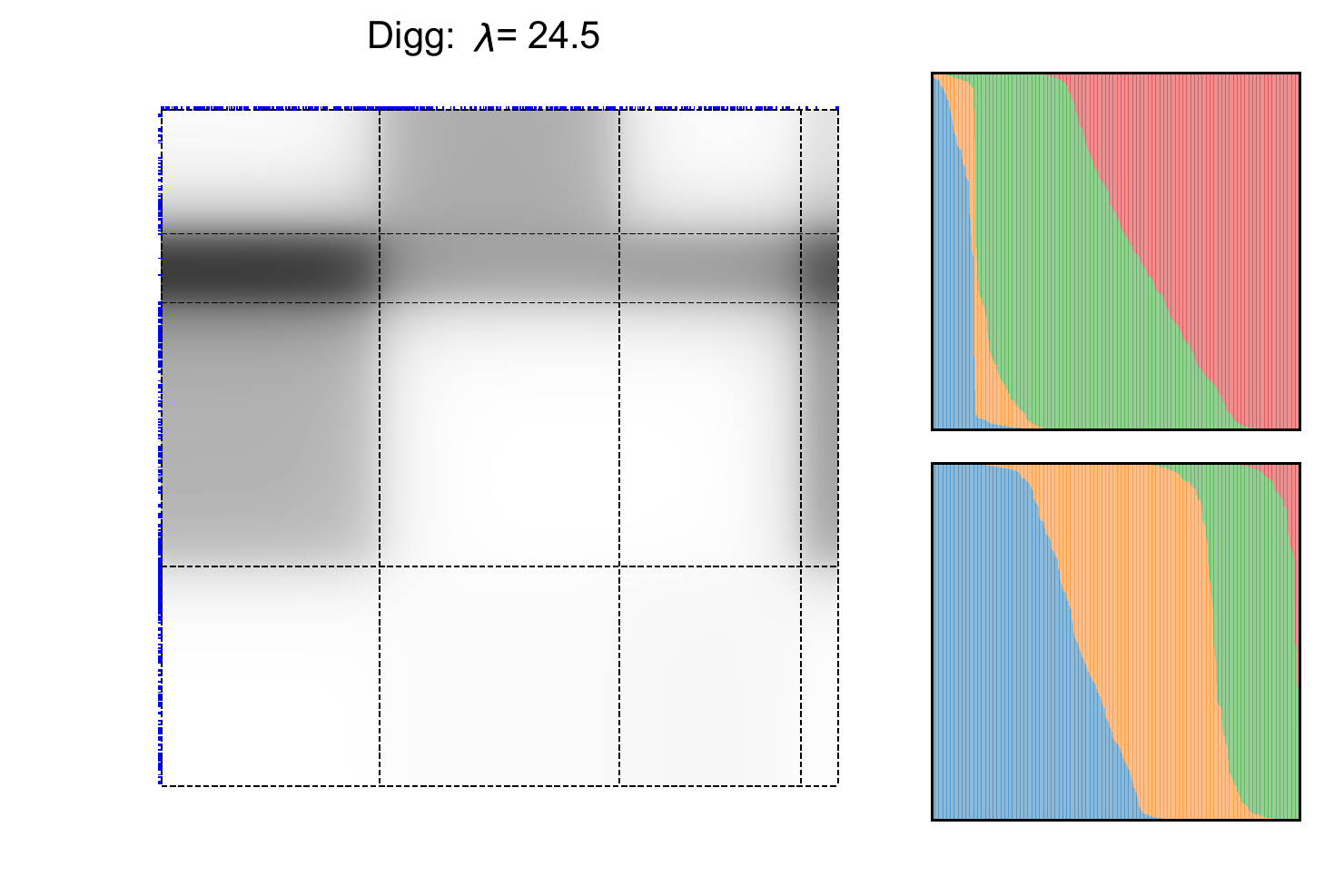}\quad
\includegraphics[width = 0.49 \textwidth]{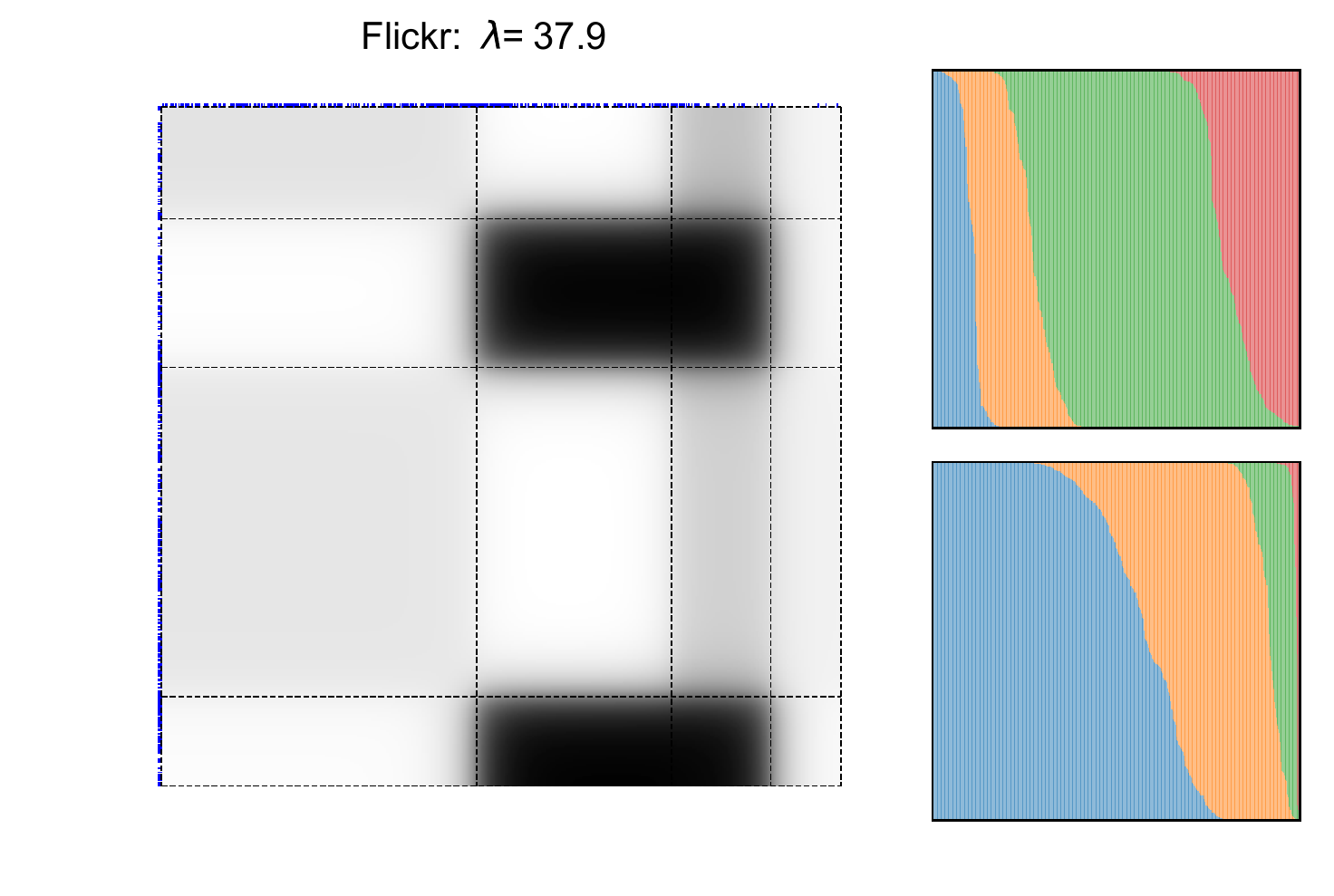}\quad
\includegraphics[width = 0.49 \textwidth]{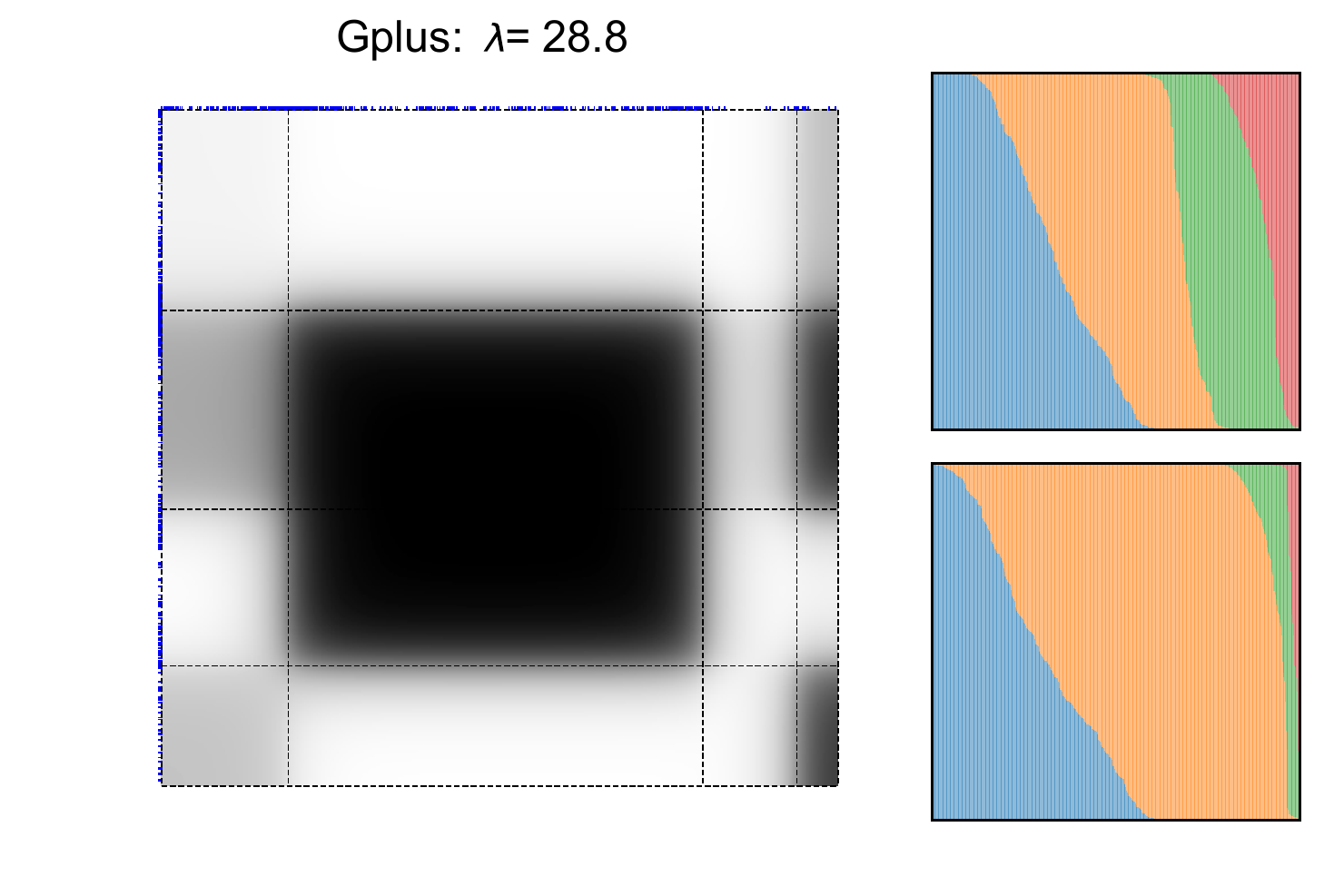}\quad
\includegraphics[width = 0.49 \textwidth]{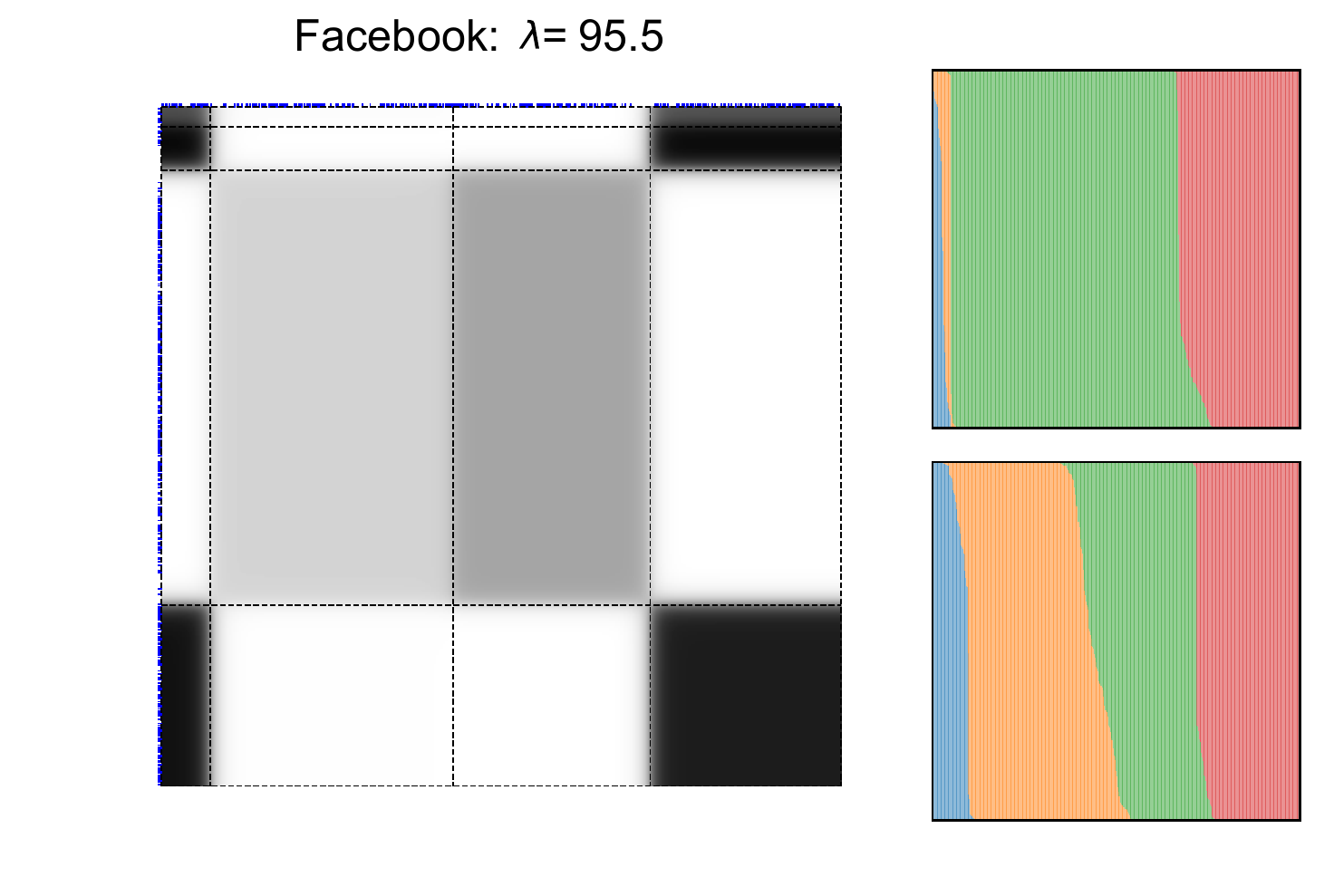}\quad
\includegraphics[width = 0.49 \textwidth]{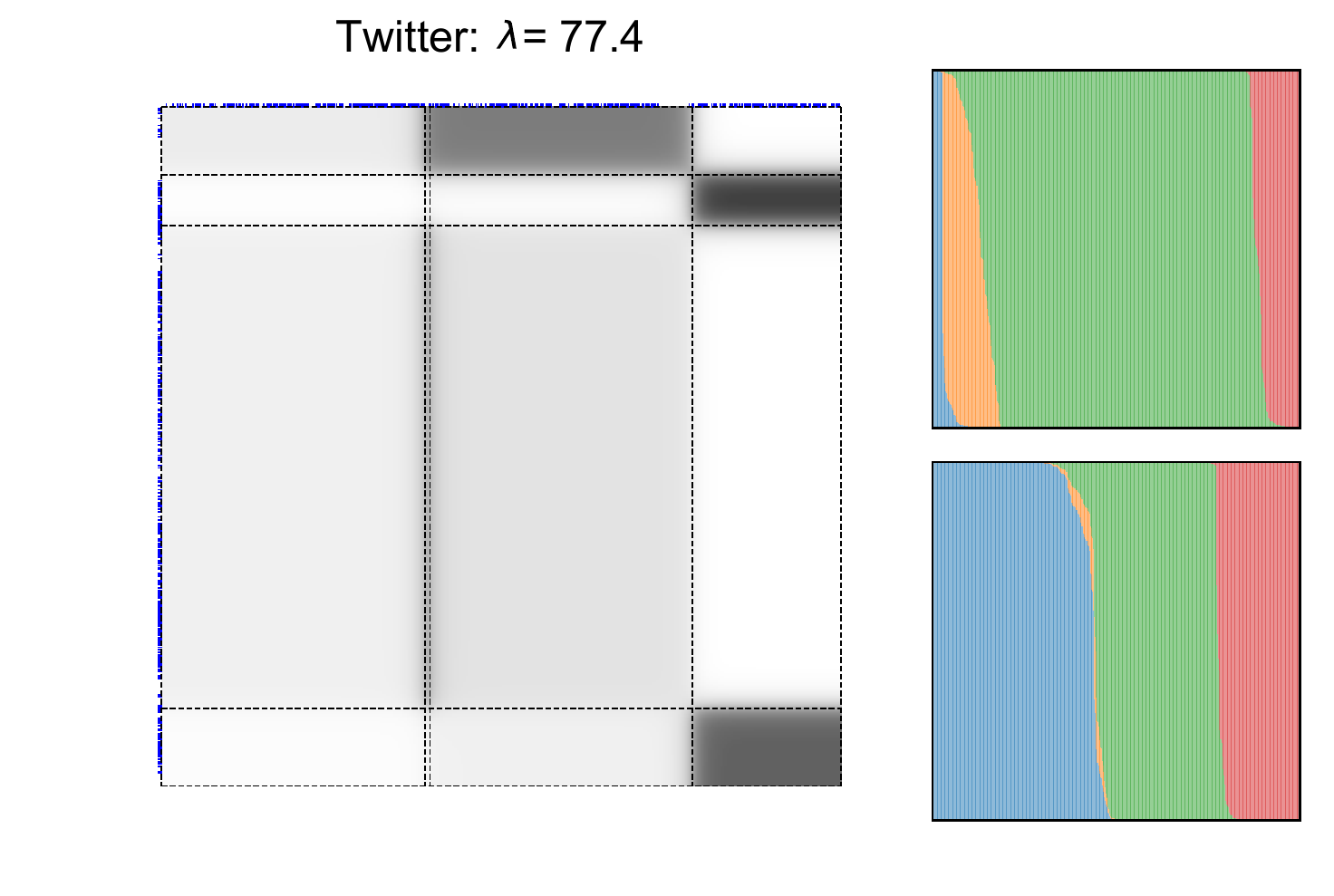}\quad
\caption{Visualisation of mixture intensity from one posterior draw, the posterior mean of smoothing parameter $\lambda$ and pairwise hidden labels on the Delicious, Digg, Flickr, Gplus, Facebook and Twitter datasets when implementing the Latent Feature Smoothing Graphon within the Stochastic Block Model. There are three figures for each dataset. Left: the grey level in the unit square illustrates the predicted mixture intensity for each relation (darker = higher intensity), the dotted lines indicate the related partition~($\boxplus=\pmb{\theta}^{(1)}\times\pmb{\theta}^{(2)}$), and the coordinates of all the nodes are visualised by the blue dots. 
Right: the different colours 
represent the different values of the latent labels $s_{ij}$~(right top) and $r_{ij}$~(right bottom), with the $x$-axis indicating different nodes (sorted by the ratio of the labels) and the $y$-axis showing the proportion of different labels for each node.}
\label{fig:graphon_visualization}
\end{figure*}

\subsection{Experimental setting} 
The hyper-parameters for each method are set as follows: for the SBM, LFSG-SBM, ISG-SBM, MMSB and MP-RM, the hyper-parameters $\alpha_0, \beta_0$ used in generating the block intensities are set as {$\alpha_0=S\%, \beta_0=1-S\%$}, where $S\%$ refers to the sparsity shown in Table~\ref{rm_dataset}, such that the block intensity has an expectation equivalent to the sparsity of the exchangeable relational data; for the SBM, LFSG-SBM, ISG-SBM and MMSB, we set the group distribution of $\pmb{\theta}$ as $\text{Dirichlet}(\pmb{1}_{1\times 4})$. Hence, the number of groups in each dimension in these models is set as $4$, with a total of $16$ blocks generated in the unit square; for the MP-RM, the budget parameter is set to $3$, which suggests that approximately $(3+1)\times(3+1)$ blocks would be generated.

\subsection{Link prediction performance}
The performance of each models in the task of link prediction is shown in Fig.~\ref{fig:auc_precision}, which reports both the average area under the curve of the receiver operating characteristic~(AUC) and the precision-recall~(Precision).
The AUC denotes the probability that the model will rank a randomly chosen positive link higher than a randomly chosen zero-valued link. The precision is the average ratio of correctly predicted positive links to the total number of predicted positive links. Higher values of AUC and precision indicate better model performance. 
For each dataset, we vary the ratio of training data from $10\%$ to $90\%$ and use the remainder for testing. The training/test data split is created in a row-wise manner. In particular, we take the same ratio of training data from each row of the relational matrix, so that each node shares the same amount of training data. 

From Fig.~\ref{fig:auc_precision}, 
 both the AUC and precision of all models improves as the amount of training data increases.
 The trend generally becomes steady when the proportion is larger than $0.3$, indicating 
 the amount of data required to fit a model with a $\sim$16-block complexity.

Except for the Facebook data, we can see that the AUC and precision of both the ISG-SBM and the LFSG-SBM are better than for the piecewise-constant graphon models (i.e.~the SBM and MP-RM). The proposed smoothing graphons typically achieve similar performance  to the MMSB,  demonstrating that the smoothing graphon strategy
is useful for improving model performance. 
For the Facebook dataset, the SBM seems to perform better than the smoothing graphon-based models. This is examined in greater detail in the next section.

\subsection{Graphon and hidden label visualisation}

In addition to the quantitative analysis, we visualise the generated graphons and hidden labels under the LFSG-SBM on all six data sets in Fig.~\ref{fig:graphon_visualization}. For each dataset, we visualise the resulting mixture intensities for one posterior sample, with the learned posterior mean of the smoothing parameter $\lambda$, based on using $90\%$ training data.
We observe that the displayed graphon intensities exhibit smooth transitions between blocks for each dataset, highlighting
that continuous, rather than discrete, mixture intensity values are generated under the smoothing graphon.
%
%
The transition speed of the intensity between blocks is influenced by the smoothing parameter $\lambda$ -- a larger value of $\lambda$ leads to a less smooth graphon, and a smaller value of $\lambda$ to a more smooth graphon --  similar to that observed in Fig.~\ref{fig:weight_function_vis}.

In Fig.~\ref{fig:graphon_visualization}, for each dataset, we also display the posterior proportions of the pairwise hidden labels $s_{ij}$ (top right) and $r_{ij}$ (bottom right) for each node. Here the $x$-axis indicates different nodes~(sorted by the label probabilities) and the $y$-axis displays the {posterior mean} of label probabilities~(each label represented by a different colour). For each node $i$ on the $x$-axis, the more colours observed on the $y$-axis indicates a greater diversity of groups associated with that node,
which in turn represents a higher potential for that node to belong to different groups when interacting with other nodes. In other words, the larger the tendency away from vertical line transitions between groups in these plots means
a larger number of nodes belonging to multiple groups. 

Compared with the value of the smoothing parameter $\lambda$ learned on the other four data sets, the values of $\lambda$ estimated from the Facebook and Twitter datasets are larger. Further, the visualisations of the hidden labels for these two data sets are partitioned by almost straight vertical lines, which suggests that only one label is a realistic possibility for most of the nodes. This could explain why both the AUC and precision values of the ISG-SBM and the LFSG-SBM are less competitive with those of the SBM on these two datasets (Fig.~\ref{fig:auc_precision}). Here, that the SBM assigns each node to exactly one group only,  which aligns well with the ground-truth for these two datasets.
%

Another explanation for the performance on the datasets of Facebook and Twitter is that we can recover the SBM if and only if $\lambda= \infty$. For any finite value of the smoothing parameter $\lambda$, it is impossible to have any posterior mass on the SBM. To this end, we might use a mapping to map $\lambda$ from $(0, \infty)\rightarrow (0, 1)$ (e.g. $1-e^{-\lambda}$) such that we are able to place posterior mass close to $1$. As the mapped value would be easily to approximate $1$, the models of ISG-SBM and LFSG-SBM could be able to perform at least as well as the SBM, even for the datasets of Facebook and Twitter. 


%

\section{Conclusion}
In this paper, we have introduced a smoothing strategy to modify conventional piecewise-constant graphons in order to increase their continuity. Through the introduction of a single smoothing parameter $\lambda$, we first developed the Integrated Smoothing Graphon~(ISG) that addresses the key limitation of existing piecewise-constant graphons which only generate a limited number of discrete  intensity values. To improve the computational efficiency of the ISG and to allow for the possibility of each node to belong to multiple groups, we further developed the Latent Feature Smoothing Graphon~(LFSG) by the introduction of auxiliary hidden labels. Our experimental results verify the effectiveness of this smoothing strategy in terms of greatly improved AUC and precision scores in the task of link prediction. The visualisations of the generated graphons and the posterior hidden label summaries further provide an intuitive understanding of the nature of the smoothing mechanism for the given dataset.


\bibliography{Xuhui_Machine_Learning}

\begin{thebibliography}{10}
\providecommand{\url}[1]{#1}
\csname url@samestyle\endcsname
\providecommand{\newblock}{\relax}
\providecommand{\bibinfo}[2]{#2}
\providecommand{\BIBentrySTDinterwordspacing}{\spaceskip=0pt\relax}
\providecommand{\BIBentryALTinterwordstretchfactor}{4}
\providecommand{\BIBentryALTinterwordspacing}{\spaceskip=\fontdimen2\font plus
\BIBentryALTinterwordstretchfactor\fontdimen3\font minus
  \fontdimen4\font\relax}
\providecommand{\BIBforeignlanguage}[2]{{%
\expandafter\ifx\csname l@#1\endcsname\relax
\typeout{** WARNING: IEEEtran.bst: No hyphenation pattern has been}%
\typeout{** loaded for the language `#1'. Using the pattern for}%
\typeout{** the default language instead.}%
\else
\language=\csname l@#1\endcsname
\fi
#2}}
\providecommand{\BIBdecl}{\relax}
\BIBdecl

\bibitem{nowicki2001estimation}
K.~Nowicki and T.~A. Snijders, ``Estimation and prediction for stochastic block
  structures,'' \emph{Journal of the American Statistical Association},
  vol.~96, no. 455, pp. 1077--1087, 2001.

\bibitem{ishiguro2010dynamic}
K.~Ishiguro, T.~Iwata, N.~Ueda, and J.~B. Tenenbaum, ``Dynamic infinite
  relational model for time-varying relational data analysis,'' in \emph{NIPS},
  2010, pp. 919--927.

\bibitem{nonpa2013schmidt}
M.~N. Schmidt and M.~M{\o}rup, ``Nonparametric {Bayesian} modeling of complex
  networks: An introduction,'' \emph{IEEE Signal Processing Magazine}, vol.~30,
  no.~3, pp. 110--128, 2013.

\bibitem{tnnls_8497036}
X.~{Zhang}, ``A nonconvex relaxation approach to low-rank tensor completion,''
  \emph{IEEE Transactions on Neural Networks and Learning Systems}, vol.~30,
  no.~6, pp. 1659--1671, 2019.

\bibitem{pensky2019dynamic}
M.~Pensky \emph{et~al.}, ``Dynamic network models and graphon estimation,''
  \emph{The Annals of Statistics}, vol.~47, no.~4, pp. 2378--2403, 2019.

\bibitem{tnnls_7112169}
X.~{Luo}, M.~{Zhou}, S.~{Li}, Z.~{You}, Y.~{Xia}, and Q.~{Zhu}, ``A nonnegative
  latent factor model for large-scale sparse matrices in recommender systems
  via alternating direction method,'' \emph{IEEE Transactions on Neural
  Networks and Learning Systems}, vol.~27, no.~3, pp. 579--592, 2016.

\bibitem{tnnls_8525418}
Q.~{Zhang}, J.~{Lu}, D.~{Wu}, and G.~{Zhang}, ``A cross-domain recommender
  system with kernel-induced knowledge transfer for overlapping entities,''
  \emph{IEEE Transactions on Neural Networks and Learning Systems}, vol.~30,
  no.~7, pp. 1998--2012, 2019.

\bibitem{Li_transfer_2009}
B.~Li, Q.~Yang, and X.~Xue, ``Transfer learning for collaborative filtering via
  a rating-matrix generative model,'' in \emph{ICML}, 2009, pp. 617--624.

\bibitem{kemp2006learning}
C.~Kemp, J.~B. Tenenbaum, T.~L. Griffiths, T.~Yamada, and N.~Ueda, ``Learning
  systems of concepts with an infinite relational model,'' in \emph{AAAI},
  vol.~3, 2006, pp. 381--388.

\bibitem{roy2009mondrian}
D.~M. Roy and Y.~W. Teh, ``The {Mondrian} process,'' in \emph{NIPS}, 2009, pp.
  1377--1384.

\bibitem{tnnls_6508899}
Y.~{Pang}, Z.~{Ji}, P.~{Jing}, and X.~{Li}, ``Ranking graph embedding for
  learning to rerank,'' \emph{IEEE Transactions on Neural Networks and Learning
  Systems}, vol.~24, no.~8, pp. 1292--1303, 2013.

\bibitem{tnnls_6842607}
D.~{Bouzas}, N.~{Arvanitopoulos}, and A.~{Tefas}, ``Graph embedded
  nonparametric mutual information for supervised dimensionality reduction,''
  \emph{IEEE Transactions on Neural Networks and Learning Systems}, vol.~26,
  no.~5, pp. 951--963, 2015.

\bibitem{tnnls_8587135}
A.~{Dutta} and H.~{Sahbi}, ``Stochastic graphlet embedding,'' \emph{IEEE
  Transactions on Neural Networks and Learning Systems}, vol.~30, no.~8, pp.
  2369--2382, 2019.

\bibitem{tnnls_6208890}
K.~{Nikolaidis}, E.~{Rodriguez-Martinez}, J.~Y. {Goulermas}, and Q.~H. {Wu},
  ``Spectral graph optimization for instance reduction,'' \emph{IEEE
  Transactions on Neural Networks and Learning Systems}, vol.~23, no.~7, pp.
  1169--1175, 2012.

\bibitem{tnnls_8440680}
Q.~{Wang}, Z.~{Qin}, F.~{Nie}, and X.~{Li}, ``Spectral embedded adaptive
  neighbors clustering,'' \emph{IEEE Transactions on Neural Networks and
  Learning Systems}, vol.~30, no.~4, pp. 1265--1271, 2019.

\bibitem{orbanz2009construction}
P.~Orbanz, ``Construction of nonparametric {Bayesian} models from parametric
  {Bayes} equations,'' in \emph{NIPS}, 2009, pp. 1392--1400.

\bibitem{orbanz2014bayesian}
P.~Orbanz and D.~M. Roy, ``Bayesian models of graphs, arrays and other
  exchangeable random structures,'' \emph{IEEE transactions on pattern analysis
  and machine intelligence}, vol.~37, no.~2, pp. 437--461, 2014.

\bibitem{RandomFunPriorsExchArrays}
J.~Lloyd, P.~Orbanz, Z.~Ghahramani, and D.~M. Roy, ``{Random function priors
  for exchangeable arrays with applications to graphs and relational data},''
  in \emph{NIPS}, 2012, pp. 1007--1015.

\bibitem{nakano2014rectangular}
M.~Nakano, K.~Ishiguro, A.~Kimura, T.~Yamada, and N.~Ueda, ``Rectangular tiling
  process,'' in \emph{ICML}, 2014, pp. 361--369.

\bibitem{NIPS2018_RBP}
X.~Fan, B.~Li, and S.~Sisson, ``Rectangular bounding process,'' in
  \emph{NeurIPS}, 2018, pp. 7631--7641.

\bibitem{hoover1979relations}
D.~N. Hoover, ``Relations on probability spaces and arrays of random
  variables,'' \emph{Preprint, Institute for Advanced Study, School of
  Mathematics, Princeton, NJ}, 1979.

\bibitem{aldous1981representations}
D.~J. Aldous, ``Representations for partially exchangeable arrays of random
  variables,'' \emph{Journal of Multivariate Analysis}, vol.~11, no.~4, pp.
  581--598, 1981.

\bibitem{airoldi2009mixed}
E.~M. Airoldi, D.~M. Blei, S.~E. Fienberg, and E.~P. Xing, ``Mixed membership
  stochastic blockmodels,'' in \emph{NIPS}, 2009, pp. 33--40.

\bibitem{pmlr-v84-fan18b}
X.~Fan, B.~Li, and S.~A. Sisson, ``The binary space partitioning-tree
  process,'' in \emph{AISTATS}, vol.~84, 2018, pp. 1859--1867.

\bibitem{roy2007learning}
D.~M. Roy, C.~Kemp, V.~Mansinghka, and J.~B. Tenenbaum, ``Learning annotated
  hierarchies from relational data,'' in \emph{NIPS}, 2007, pp. 1185--1192.

\bibitem{roy2011thesis}
D.~M. Roy, ``Computability, inference and modeling in probabilistic
  programming,'' Ph.D. dissertation, MIT, 2011.

\bibitem{pmlr-v89-fan18a}
X.~Fan, B.~Li, and S.~A. Sisson, ``The binary space partitioning forests,'' in
  \emph{AISTATS}, vol.~89, 2019, pp. 3022--3031.

\bibitem{andrieu2010particle}
C.~Andrieu, A.~Doucet, and R.~Holenstein, ``Particle markov chain monte carlo
  methods,'' \emph{{Journal of the Royal Statistical Society: Series B
  (Statistical Methodology)}}, vol.~72, no.~3, pp. 269--342, 2010.

\bibitem{LakOryTeh2015ParticleGibbs}
B.~Lakshminarayanan, D.~M. Roy, and Y.~W. Teh, ``Particle {Gibbs} for
  {Bayesian} additive regression trees,'' in \emph{AISTATS}, 2015, pp.
  553--561.

\bibitem{Zafarani+Liu:2009}
R.~Zafarani and H.~Liu, ``Social computing data repository at {ASU},'' 2009.

\bibitem{leskovec2012learning}
J.~Leskovec and J.~J. Mcauley, ``Learning to discover social circles in ego
  networks,'' in \emph{Advances in neural information processing systems},
  2012, pp. 539--547.

\end{thebibliography}
\bibliographystyle{IEEEtran}

\end{document}